\documentclass{article}

\usepackage[nonatbib,final]{neurips_2025}
\usepackage[utf8]{inputenc} % allow utf-8 input
\usepackage[T1]{fontenc}    % use 8-bit T1 fonts
\usepackage{hyperref}       % hyperlinks
\usepackage{url}            % simple URL typesetting
\usepackage{booktabs}       % professional-quality tables
\usepackage{amsfonts}       % blackboard math symbols
\usepackage{nicefrac}       % compact symbols for 1/2, etc.
\usepackage{microtype}      % microtypography
\usepackage{xcolor}         % colors

\usepackage{algorithm} % for algorithm environment
\usepackage{algorithmic} % for algorithm environment
\usepackage[numbers,sort&compress]{natbib}
\usepackage{amsmath}
\usepackage{amssymb}
\usepackage{mathtools}
\usepackage{amsthm}
\theoremstyle{plain}
\newtheorem{theorem}{Theorem}[section]
\newtheorem{proposition}[theorem]{Proposition}
\newtheorem{lemma}[theorem]{Lemma}
\newtheorem{corollary}[theorem]{Corollary}
\theoremstyle{definition}
\newtheorem{definition}[theorem]{Definition}

\theoremstyle{remark}

\usepackage{tikz}

\usepackage{wrapfig}
\usepackage{graphicx}
\usepackage{caption}
\usepackage{threeparttable}

\usepackage{siunitx}
\DeclareSIUnit{\cells}{cells}

\usepackage[capitalize,noabbrev]{cleveref}
\usepackage[inline]{enumitem}
\usepackage{chngcntr} % for \counterwithin

\newcommand{\lolbo}{\texttt{LOL-BO}}
\newcommand{\robot}{\texttt{ROBOT}}
\newcommand{\turbo}{\texttt{TuRBO}}
\newcommand{\morbo}{\texttt{MORBO}}

\newcommand{\cluso}{\texttt{CluSO}}
\newcommand{\ourmethod}{\texttt{MOCOBO}}
\newcommand{\ourmethodfullname}{Multi-Objective Coverage Bayesian Optimization (\texttt{MOCOBO})}
\newcommand{\divf}{\delta} 

\DeclareMathOperator*{\argmax}{arg\,max}

\usepackage{amsfonts}
\newcommand{\bx}{\mathbf{x}}

\newcommand{\bp}{\mathbf{p}}
\newcommand{\bb}{\mathbf{b}}
\newcommand{\by}{\mathbf{y}}

\newcommand{\inputdom}{\mathcal{X}}
\newcommand{\tr}[1]{\mathcal{T}_{#1}}

\title{Multi-Objective Coverage Bayesian Optimization (\texttt{MOCOBO})}

\author{%
  Natalie Maus\thanks{Correspondence to: \href{mailto:nmaus@seas.upenn.edu}{nmaus@seas.upenn.edu}}\quad\;%\\
  Kyurae Kim \quad\;%\\
  Yimeng Zeng \quad\;%\\
  Haydn Thomas Jones \quad\;%\\
  Fangping Wan\\
  \textbf{Marcelo Der Torossian Torres}\quad\;%\\
  \textbf{Cesar de la Fuente-Nunez}\quad\;%\\
  \textbf{Jacob R. Gardner}\\
  \\
  University of Pennsylvania\\
  Philadelphia, PA, USA\\
}

\begin{document}

\maketitle

\begin{abstract}
In multi-objective black-box optimization, the goal is typically to find solutions that optimize a set of $T$ black-box objective functions, $f_1, \ldots f_T$, simultaneously. Traditional approaches often seek a single Pareto-optimal set that balances trade-offs among all objectives. 
In contrast, we consider a problem setting that departs from this paradigm: finding a small set of $K < T$ solutions, that collectively ``cover'' the $T$ objectives. A set of solutions is defined as ``covering'' if, for each objective $f_1, \ldots f_T$, there is at least one good solution. 
A motivating example for this problem setting occurs in drug design. 
For example, we may have $T$ pathogens and aim to identify a set of $K < T$ antibiotics such that at least one antibiotic can be used to treat each pathogen. 
This problem, known as coverage optimization, has yet to be tackled with the Bayesian optimization (BO) framework.
To fill this void, we develop \ourmethodfullname{}, a BO algorithm for solving coverage optimization. 
Our approach is based on a new acquisition function reminiscent of expected improvement in the vanilla BO setup.
We demonstrate the performance of our method on high-dimensional black-box optimization tasks, including applications in peptide and molecular design.
Results show that the coverage of the $K < T$ solutions found by \ourmethod{} matches or nearly matches the coverage of $T$ solutions obtained by optimizing each objective individually.
Furthermore, in \textit{in vitro} experiments, the peptides found by \ourmethod{} exhibited high potency against drug-resistant pathogens, further demonstrating the potential of \ourmethod{} for drug discovery.
All of our code is publicly available at the following link: \url{https://github.com/nataliemaus/mocobo}.
\end{abstract}

%The \LaTeX{} style file contains three optional arguments: \verb+final+, which creates a camera-ready copy, \verb+preprint+, which creates a preprint for submission to, e.g., arXiv, and \verb+nonatbib+, which will not load the \verb+natbib+ package for you in case of package clash.
% final is for accepted papers only 

% \paragraph{Preprint option}
% If you wish to post a preprint of your work online, e.g., on arXiv, using the NeurIPS style, please use the \verb+preprint+ option. This will create a nonanonymized version of your work with the text ``Preprint. Work in progress.'' in the footer. This version may be distributed as you see fit, as long as you do not say which conference it was submitted to. 

% At submission time, please omit the \verb+final+ and \verb+preprint+ options. 

% 9 page limit 
\section{Introduction}
\label{intro}
\begin{wrapfigure}{r}{0.5\textwidth}
\vskip -0.54in
  \centering
  \includegraphics[width=0.48\textwidth]{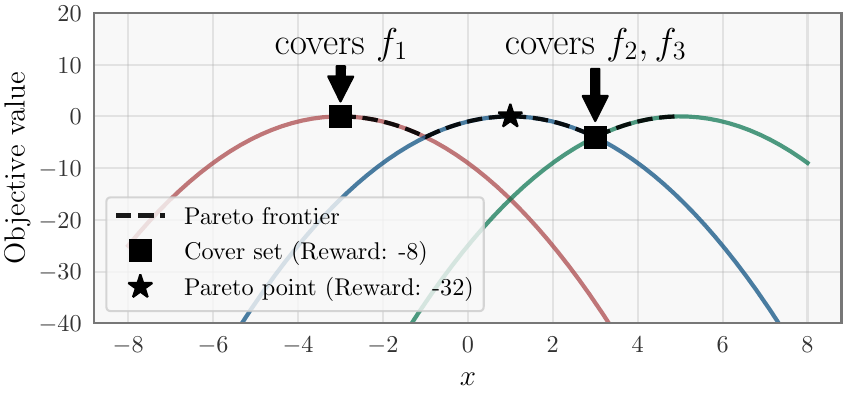}
\vskip -1ex
  \caption{Traditional multi-objective optimization for $T$ objectives might select any point along the Pareto frontier, but in some situations like this any Pareto optimal point performs poorly on at least one objective. In situations where multiple $K < T$ solutions are allowed ($\blacksquare$), we can sometimes optimize all objectives well. Note that this is a simplified schematic meant to illustrate intuition. 
  }
  \label{fig:example-1d}
\vspace{-8ex}
\end{wrapfigure}
Bayesian optimization (BO; \citep{jones1998efficient,shahriari2015taking,garnett2023bayesian}) is a general framework for sample-efficient optimization of black-box functions.
By using a probabilistic surrogate model, such as a Gaussian process (GP; \citep{rasmussen2005gaussian}), Bayesian optimization balances exploration and exploitation to identify high-performing solutions with a limited number of function evaluations.
BO has been successfully applied in a wide range of domains, including hyperparameter tuning~\cite{snoek2012practical, turner2021bayesian}, A/B testing~\cite{letham2019noisyei}, chemical engineering~\cite{hernandez2017parallel}, drug discovery~\cite{negoescu2011knowledge}, and more.

For application domains such as drug discovery, however, the setup assumed by vanilla BO, where there is only one clear objective, is often insufficient. 
Consider the example of designing antibodies for treating a broad spectrum of pathogens.
Here, the activity against each individual pathogen serves as an objective, each equally important.
Ideally, one would seek to optimize all objectives simultaneously by obtaining a single broad-spectrum antibiotic.
When this is not possible, however, we should seek antibiotics that are potent against a large enough group of pathogens.
By identifying a set of such solutions---a \textit{covering set}---we would obtain a set of drugs that are \textit{together} lethal against the set of pathogens.

At first, this setup may be reminiscent of multi-objective optimization, which has been tackled under the black-box setup via multi-objective BO (MOBO; \citep{mobo1, mobo2, mobo3, dgemo, morbo, lambo, mobo4-usemo}).
These methods search for Pareto-optimal solutions that balance the trade-offs between the individual objectives, enabling practitioners to \textit{post hoc} select a \textit{single} solution lying on the Pareto front based on their preferences. 
In cases where some objectives are completely incompatible with each other, however, selecting a single solution on the Pareto front cannot simultaneously satisfy performance requirements across those objectives.
This is illustrated in \cref{fig:example-1d}, where no individual solution can provide acceptable performance across all objectives due to the extreme trade-offs between objectives.

Instead, our problem is distinctly described as the \textit{coverage optimization} problem.
More formally, consider a set of $T$ distinct objectives.
Our goal is to obtain a set-valued solution of size $K < T$, where the set ``covers'' the $T$ objectives:
For each objective, this $K$-element covering set should contain at least one solution that performs similarly to the maximizer of that objective.
In the context of drug design, this translates to identifying $K$ drugs such that each of the $T$ pathogens is effectively addressed by at least one drug in the set. 
This setting is particularly relevant to applications where the size of $K$ is associated with downstream cost.
For instance, in drug discovery, the number of compounds that can be synthesized is fundamentally constrained by the capacity of the facility.

Previously, the coverage optimization problem has been tackled in~\citep{icml25-r1-1,icml25-r1-2,icml25-r1-3} through gradient-based algorithms, while 
for the black-box setup, \citet{cluso} proposed an evolutionary algorithm.
Over the years, variants of BO have demonstrated superior performance over evolutionary algorithms in high-dimensional~\citep{turbo,gonzalezduque2024survey,lambo} and multi-objective problems~\citep{morbo}, which has been instrumental for its success in drug discovery~\citep{gonzalezduque2024survey,apexgo,wu20205generative,khan2023realworld}.
Therefore, it is natural to ask how we could harness the strengths and flexibility of BO to attack the coverage optimization problem.

In this work, we develop \ourmethodfullname{}---a BO-based method for solving the coverage optimization problem with high-dimensional black-box objectives.
\ourmethod{} primarily consists of a coverage-optimization analog of the celebrated expected improvement (EI; \citep{mockus1982bayesian}) acquisition function, where the covering set is greedily constructed via submodular optimization~\citep{greedysubmodular}.
We evaluate the resulting algorithm on realistic optimization tasks: drug discovery over molecules and peptides, and tuning of image processing pipelines.
Baselines include the method of \citep{cluso} and MOBO.

Our contributions are summarized as follows:
\begin{enumerate}[itemsep=0ex,leftmargin=3ex,label=$\bullet$]
    \item We propose \ourmethod{}, a coverage optimization algorithm for black-box objectives that can deal with structured, high-dimensional search spaces, and scale to large numbers of function evaluations. 
    
    \item We experimentally validate \ourmethod{} on challenging, high-dimensional optimization tasks including structured drug discovery over molecules and peptides, demonstrating its ability to consistently outperform state-of-the-art BO methods in identifying high-performing covering sets of solutions. 
    
    \item We demonstrate empirically that the small sets of $K<T$ solutions found by \ourmethod{} consistently nearly match the performance of a larger set of $T$ solutions individually optimized for each task. 
    
    \item We demonstrate the potential of \ourmethod{} for drug discovery using an \textit{in vitro} experiment showing it produces potent antimicrobial peptides that cover 9 out of 11 drug resistant or otherwise challenging to kill pathogens, with moderate activity on 1 out of 11.
\end{enumerate}
\vspace{-1ex}
\section{Background}\label{sec: background}
\label{background}
\vspace{-1ex}
\paragraph{Bayesian Optimization.} 
To solve the black-box optimization problem, BO operates by iteratively selecting a query according to a certain search policy~\citep{garnett2023bayesian} and observing the objective value on the query.
The observations that have been obtained up to iteration $n \geq 0$ are collected into a dataset $D_n$ to form a surrogate model, typically GP, which supplies information to the policy.
Most policies are chosen to maximize a function known as \textit{acquisition function}~\citep{garnett2023bayesian} such that, at each step $n \geq 0$, the query $\mathbf{x}_{n+1}$ is found as
{%
\setlength{\belowdisplayskip}{1ex} \setlength{\belowdisplayshortskip}{1ex}
\setlength{\abovedisplayskip}{1ex} \setlength{\abovedisplayshortskip}{1ex}
\begin{align}
    \mathbf{x}_{n+1} = \argmax_{\mathbf{x} \in \mathcal{X}} \alpha\left(\mathbf{x}; D_n\right) \; , 
    \label{eq:acquisition}
\end{align}
}%
where EI~\citep{mockus1982bayesian} is a popular example of an acquisition function.
In our case, the optimization problem is over covering sets.
Naturally, an acquisition function for this setting is needed.

\vspace{-1ex}
\paragraph{Trust Region Bayesian Optimization (TuRBO).} 
For high-dimensional problems, the basic scheme of solving \cref{eq:acquisition} is no longer effective.
Local BO methods~\citep{turbo} address this issue by restricting the search space to a trust region that is adjusted adaptively.
This simple modification has been shown to be effective and has been employed in various setups.
In high-dimensional MOBO problems specifically, a variant of TuRBO~\citep{turbo} referred to as \turbo{}-$M$ has been shown to be effective~\citep{morbo,robot}.
\turbo{}-$M$ works by running $M$ parallel local optimization runs, each maintaining its own dataset $D_{i}$ and surrogate model. 
Each local optimizer proposes candidates within a hyper-rectangular trust region $\tr{i}$, which is taken to be a rectangular subset of the space $\mathcal{X}$ centered on the incumbent $\bx^{+}_i$. 
The side of each $\tr{i}$ is chosen to have a length $\ell_i \in [\ell_{\mathrm{min}}, \ell_{\mathrm{max}}]$. 
If an optimizer improves the incumbent for $\rho_{\mathrm{succ}}$ consecutive iterations, $\ell_i$ expands to $\min(2\ell_i, \ell_{\mathrm{max}})$. If not improved in $\rho_{\mathrm{fail}}$ iterations, $\ell_i$ is halved. 
Optimizers are restarted if $\ell_i$ drops below $\ell_{\mathrm{min}}$. 
We will later incorporate this scheme for solving the coverage optimization problem.

\section{Multi-Objective Coverage Bayesian Optimization}
\label{sec: methods}

\subsection{Coverage Optimization}

Consider $T > 1$ objectives $f_1, \ldots, f_T$, each defined as $f_t: \inputdom \rightarrow \mathbb{R}$ for each $t = 1, \ldots, T$, all sharing the same input domain $\inputdom$.
For some user-defined parameter $K \in \{1, \ldots, T - 1\} $, we consider the task of finding a set of $K$ solutions $S^{*} = \left\{ \bx^{*}_{1}, \ldots, \bx^{*}_{K} \right\}$ that ``covers'' the $T$ objectives ${\{f_t\}}_{t = 1, \ldots, T}$.
Specifically, the set $S^{*} \subseteq \inputdom$ ``covers'' the $T$ objectives if, for each objective $f_t \in \{f_1, \ldots, f_T\}$, there is at least one solution $\bx^{*}_{k} \in S^{*}$ for which $f_t$ is well optimized by $\bx^{*}_{k}$. We evaluate how well a set of $K$ points $\left\{ \bx_{1}, \ldots, \bx_{K} \right\}$ ``covers" the $T$ objectives using the following ``coverage score''
Formally,
\begin{align}
c(\left\{ \bx_{1}, \ldots, \bx_{K} \right\}) = 
{\sum}_{t = 1}^{T} \max_{k = 1}^{K} f_t(\bx_k). \label{eq:coverage_score}
\end{align}
Formally, we seek a set $S^{*} := \left\{ \bx^{*}_{1}, \ldots, \bx^{*}_{K} \right\}$ such that: 
\begin{align}
S^{*} &= \argmax_{ \left\{ \bx_{1}, \ldots, \bx_{K} \right\} \subseteq \inputdom} c(\left\{ \bx_{1}, \ldots, \bx_{K} \right\}). \label{eq:problem_def}
\end{align}
The coverage score in \eqref{eq:coverage_score} only credits a single solution for each objective. 
If, for example, $f_{t}$ is maximized by one of the $\bx_{i}$ in the solution set, improving the value of $f_{t}$ on some other $\bx_{j \neq i}$ becomes irrelevant so long as $f_{t}(\bx_{i}) \geq f_{t}(\bx_{j})$. 
In the setting where $K=1$, the coverage score collapses into a trivial linearization of the objectives, and true multi-objective BO methods should be preferred. In the setting where $K=T$, the coverage score is trivially optimized by maximizing each objective independently.
This problem has previously been formulated in~\citep{icml25-r1-1,icml25-r1-2,icml25-r1-3,cluso}, where we will proceed by developing a new BO-based solution.

\vspace{-1ex}
\subsection{Multi-Objective Coverage Bayesian Optimization (\ourmethod{})}
\label{sec:ours}
In this section, we propose \ourmethod{} - an algorithm which extends Bayesian optimization to the problem setting above. On each step of optimization $s$, we use our current set of all data evaluated so far $D_s = \left\{ (\bx_1, \by_1), \ldots, (\bx_n, \by_n)  \right\}$ to define the best covering set $S^{*}_{D_s} = \big\{ \bx^{*(s)}_{1}, \ldots, \bx^{*(s)}_{K} \in D_s \big\}$ found so far. 
Here $\by_i = (f_1(\bx_i), \ldots, f_T(\bx_i))$ and $n$ is the number of data points evaluated so far at step $s$. 
Following \cref{eq:problem_def}, we define $S^{*}_{D_s}$ as follows: 
\begin{align}
S^{*}_{D_s} 
\quad=\quad \argmax_{ \left\{ \bx_{1}, \ldots, \bx_{K} \right\} \subseteq D_s} c(\left\{ \bx_{1}, \ldots, \bx_{K} \right\}) 
\quad=\quad \argmax_{ \left\{ \bx_{1}, \ldots, \bx_{K} \right\} \subseteq D_s} {\sum}_{t = 1}^{T} \max_{k = 1}^{K} f_t(\bx_k). 
\label{eq:problem_def_each_step}
\end{align}
Here, $c(S^{*}_{D_s})$ denotes the best coverage score found by the optimizer after optimization step $s$.  

\vspace{-1ex}
\subsubsection{Candidate Selection with Expected Coverage Improvement (ECI)}
\label{sec:eci}
Given $S^{*}_{D_s}$, our surrogate model's predictive posterior $p(y \mid \bx, D)$ induces a posterior belief about the improvement in coverage score achievable by choosing to evaluate at $\bx$ next. 
We naturally extend the typical expected improvement (EI; \citep{mockus1982bayesian}) acquisition function to the coverage optimization setting by defining expected coverage improvement (ECI):
\begin{align}
\mathrm{ECI}(\bx) &= \mathbb{E}_{p(y | \bx,D)}[ \max(0,  c(S^{*}_{D_s \cup \left\{(\bx, \by)\right\}}) - c(S^{*}_{D_s}))].
\label{eq:eci}
\end{align}
Here $c(S^{*}_{D_s})$ is the coverage score of the best possible covering set from among all data observed $D_s$ -- as we shall see, constructing this set $S^{*}_{D_s}$ will be our primary challenge. $c(S^{*}_{D_s \cup \left\{(\bx, \by)\right\}})$ is the coverage score of the best possible covering set after adding the observation $(\bx, \by)$. 
Thus, ECI gives the expected improvement in the coverage score after making an observation at point $\bx$.
We aim to select points during acquisition that maximize ECI. 
Note that the name ECI is shared with a method proposed by \citet{r1-cite1} for multi-objective experimental design, but their method targets an entirely different notion of ``coverage.''
(See \cref{sec: related_works} for more details.)

We estimate ECI using a Monte Carlo (MC) approximation. 
To select a single candidate $\hat{\bx}$, we sample $m$ points $P = \left\{ \bp_{1}, \bp_{2}, ..., \bp_{m} \right\}$. 
For each sampled point $\bp_{j}$, we sample a realization $\hat{\by}_{j} = (\hat{f}_1(\bp_{j}), \ldots, \hat{f}_T(\bp_{j}))$ from the GP surrogate model posterior. 
We leverage these samples to compute an MC approximation to the ECI of each $\bp_{j}$: 
\begin{align}
\mathrm{CI}(\bp_{j}) &= \max(0,  c(S^{*}_{D_s \cup \left\{(\bp_{j}, \hat{\by}_{j})\right\}} - c(S^{*}_{D_s}))  ).
\label{eq:coverage_improvement}
\end{align}
Here, $c(S^{*}_{D_s \cup \left\{(\bp_{j}, \hat{\by}_{j})\right\}})$ is the approximation of the coverage score of the new best covering set if we choose to evaluate candidate $\bp_{j}$, assuming the candidate point will have the sampled objective values $\hat{\by}_{j}$.
We select and evaluate the candidate $\bp_{j}$ with the largest expected coverage improvement. 

\subsubsection{Greedy Approximation of Best Observed Covering Set}
\label{sec:greedy}
\begin{wrapfigure}{r}{0.51\textwidth}
  \begin{minipage}{0.49\textwidth}
    \centering
    \rule{\linewidth}{0.4pt}  
    \captionof{algorithm}{Greedy $(1 - \frac{1}{e})$-Approximation for Finding $S^{*}_{D_s}$ (Incremental Strategy)}  
    \label{alg:greedy-simple}
    \begin{algorithmic}
      \REQUIRE Dataset $D_s$, observed values $f_t(\bx)$ for all $t \in \{1, \ldots, T\}$ and $\bx \in D_s$, set size $K$.

\STATE Initialize $A \gets \emptyset$. \COMMENT{Start with an empty set.}
\STATE Compute the initial coverage score for $A$: $c(A) \gets 0$.

\FOR{$k = 1$ to $K$}
    \STATE Initialize $\bx_{\text{best}} \gets \text{None}$ and $\Delta_{\text{best}} \gets -\infty$. 
    \FOR{$\bx \in D_s \setminus A$}
        \STATE Compute marginal coverage of adding $\bx$ to $A$: 
        \[
        \Delta c \gets \sum_{t=1}^T \max \left(\max_{\bx' \in A} f_t(\bx'), f_t(\bx)\right) - c(A).
        \]
        \IF{$\Delta c > \Delta_{\text{best}}$}
            \STATE $\bx_{\text{best}} \gets \bx$ and $\Delta_{\text{best}} \gets \Delta c$.
        \ENDIF
    \ENDFOR
    \STATE Update $A \gets A \cup \{\bx_{\text{best}}\}$, $c(A) \gets c(A) + \Delta_{\text{best}}$. 
\ENDFOR

\STATE \textbf{Output:} $A^{*}_{D_s} \gets A$.
    \end{algorithmic}
    \rule{\linewidth}{0.4pt} 
  \end{minipage}
  \vspace{-25pt}
\end{wrapfigure}

The candidate acquisition method in \cref{sec:eci} utilizes the best covering set of $K$ points among all data collected so far, $S^{*}_{D_s}$. 
On each step of optimization $s$, \ourmethod{} must therefore construct $S^{*}_{D_s}$ from all observed data $D_s$, as $S^{*}_{D_s}$ may change on each step of optimization after new data is added to $D_s$.

\begin{lemma}[NP-hardness of Optimal Covering Set]
\label{lemma:np_hardness} 
Let $T, K$ be finite positive integers such that $K < T$. Let $f_1, \ldots, f_T$ be real valued functions. Let $D_s = \left\{ (\bx_1, \by_1), \ldots, (\bx_n, \by_n)  \right\}$ be a dataset of $n$ real valued data points such that for all $\bx_i$, $\by_i = (f_1(\bx_i), \ldots, f_T(\bx_i))$.  
Let $S^{*}_{D_s}$ be the optimal covering set of size $K$ in $D_s$ as defined in \cref{eq:problem_def_each_step}. Then, constructing $S^{*}_{D_s}$ is NP-hard.
\end{lemma}
\vspace{-1ex}
\begin{proof}
The proof is in \cref{sec: nphard}. 
\end{proof}
\vspace{-1ex}
Since constructing $S^{*}_{D_s}$ is NP-hard, we use an approximate construction of $S^{*}_{D_s}$ on each step of optimization $s$. 
We present \cref{alg:greedy-simple}, which is based on greedy submodular optimization~\citep{greedysubmodular},
that provides a $(1 - \frac{1}{e})$-approximation of $S^{*}_{D_s}$ after an execution time of $O(n K T)$. 

\begin{theorem}
\label{theorem:greedy-simple-approx} 
The set $A^{*}_{D_s}$ output by \cref{alg:greedy-simple} satisfies
$c(A^{*}_{D_s}) \geq \left(1 - \frac{1}{e}\right)c(S^{*}_{D_s})$.
\end{theorem}
\vspace{-2ex}
\begin{proof}
The proof is in \cref{sec: greedyproof}. 
\end{proof}
\vspace{-1ex}

\vspace{-1ex}
\paragraph{Time Complexity of \cref{alg:greedy-simple}.}
The algorithm iterates $K$ times to construct the covering set. In each iteration, it evaluates at most $n$ candidate points, and for each candidate, it computes the incremental coverage score by iterating over $T$ objectives. 
The execution time complexity is thus $O(K \cdot n \cdot T)$. 
For practical applications where we can assume relatively small $K$ and $T$, the execution time is approximately $O(n)$. 
For an empirical evaluation of the execution time, see \cref{sec:wallclock}. 

\vspace{1ex}
\begin{corollary}[\cref{alg:greedy-simple} is the best possible approximation of $S^{*}_{D_s}$]
\label{corollary:greedy-is-best} 
There is no polynomial execution time algorithm that provides a better approximation ratio unless $P = NP$. 
\end{corollary}
\vspace{-1ex}
\begin{proof}
This follows from the fact that we reduced from Max $k$-cover, for which a better approximation ratio is not practically achievable unless $P = NP$ \citep{set-cover-lnn}. 
\end{proof}

\vspace{-1ex}
\subsubsection{Extending ECI to the Batch Acquisition Setting (q-ECI)}
\label{sec:qeci}
In batch acquisition, we select a batch of $q > 1$ candidates for evaluation.
Following recent work on batch EI \citep{wilson2018maximizing, qei2}, we define q-ECI, a natural extension of ECI to the batch setting: 
\begin{align}
\label{eq:full-q-eci}
q\text{-ECI}(\mathbf{X}) &= \mathbb{E}_{p(\mathbf{Y} | \mathbf{X}, D)}
\Big[ \max\Big\{\, 0, \, \max_{r=1, \dots, q} c\big(S^{*}_{D_s \cup \left\{(\bx_r, \by_r)\right\}} \big) - c\left(S^{*}_{D_s}\right) \,\Big\} \Big].
\end{align}
q-ECI gives the expected improvement in coverage score after simultaneously observing the batch of $q$ points $\mathbf{X} = \left\{\bx_1, \ldots, \bx_q \right\}$.
However, the resulting Monte Carlo expectation would require $O(q \times m)$ evaluations of \cref{alg:greedy-simple}. We therefore adopt a more approximate batching strategy for practical use with large $q$ (see \cref{sec:approx-q-eci} for details).

\vspace{-1ex}
\subsubsection{\ourmethod{} with Trust Regions}
In order to find a set of $K$ solutions, as explained in \cref{sec: background}, \ourmethod{} follows the \turbo{}-$M$ principle.
That is, it maintains $K$ simultaneous local optimization runs using $K$ individual trust regions. 
Each local run $k \in \left\{1, \ldots, K \right\}$ aims to find a single solution $\bx^{*}_{k}$, which together form the desired set $S^{*}$. 
As in the original \turbo{} paper, trust regions are rectangular regions of the search space $\tr{k} \subseteq \mathcal{X}$ defined solely by their size and center point.
We center the $K$ trust regions on the best covering set of solutions observed so far during optimization. 
In particular, on each optimization step $s$, we construct an approximation of $S^{*}_{D_s}$ using \cref{alg:greedy-simple}, and center each trust region $\tr{k}$ on the corresponding point $\bx^{*(s)}_{k}$ in $S^{*}_{D_s} = \big\{ \bx^{*(s)}_{1}, \ldots, \bx^{*(s)}_{K} \big\}$. 
We then use our proposed ECI (or q-ECI when $q > 1$) acquisition function to select $q$ candidates for evaluation from within each trust region. 
The MC approximation of ECI described in \cref{sec:eci} can be straightforwardly applied to select candidates in trust region $\tr{k}$ by sampling the $m$ discrete points from within the rectangular bounds of $\tr{k}$. 
Since we select and evaluate $q$ candidates from each of the $K$ trust regions, the total number of observed data points $n$ at step $s$ of \ourmethod{} is $n = s \times K \times q$.
As in the original \turbo{} algorithm, each trust region $\tr{k}$ has success and failure counters that dictate the size of the trust region. For \ourmethod{}, we count a success for trust region $\tr{k}$ whenever $\tr{k}$ proposes a candidate on step $s$ that improves upon the best coverage score and is included in $S^{*}_{D_{s+1}}$.

\section{Experiments}
\label{sec:experiments}
\begin{figure*}[t]
\vspace{-2ex}
\begin{center}
\centerline{\includegraphics[width=1.0\columnwidth]{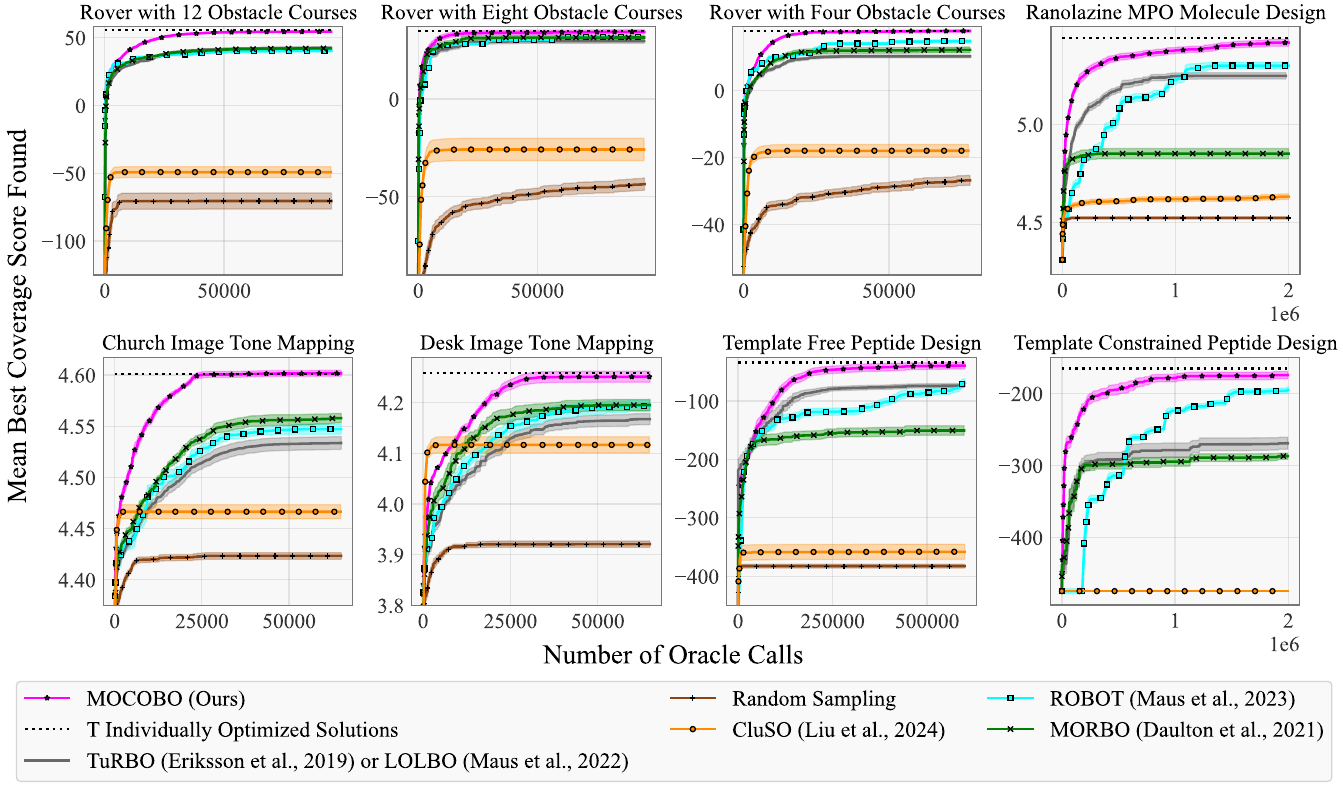}}
\caption{Coverage optimization results on all tasks considered. 
}
\label{fig:main}
\end{center}
\vskip -0.3in
\end{figure*}
We evaluate \ourmethod{} on four high-dimensional, multi-objective BO tasks for which finding a set of $K<T$ solutions to cover the $T$ objectives is desirable. Detailed descriptions of each task are in \cref{sec: tasks}. Two tasks involve continuous search spaces, allowing direct application of \ourmethod{}, while the other two involve structured spaces (molecules and peptides), requiring an extension for structured optimization.

\vspace{-1.5ex}
\paragraph{Implementation details and hyperparameters.}
We implement \ourmethod{} using BoTorch~\cite{balandat2020botorch} and GPyTorch~\cite{gardner2018gpytorch}. Code to reproduce \ourmethod{} results on all tasks considered is available on GitHub: \url{https://github.com/nataliemaus/mocobo}.
We use an acquisition batch size of $20$ for all tasks and across all BO methods compared.
Since we consider challenging high-dimensional tasks that require a large number of function evaluations, we use approximate GP surrogate models, specifically PPGPR~\cite{PPGPR}.
Further implementation details are provided in \autoref{sec:detials}.

\paragraph{Structured Search Spaces.} 
Some of our optimization tasks operate on structured search spaces that are not Euclidean.
For this, it is typical to employ a generative model\textemdash{}such as a variational autoencoder (VAE; \citealp{kingma2013auto,rezende2014stochastic})\textemdash{}to convert structured inputs into a continuous latent space for BO \cite{eissman2018bayesian, r2-cite5-lsbo, r2-cite10-lsbo-cobo, r2-cite7-lsbo-gvae, r2-cite8-lsbo, Weighted_Retraining, Huawei, siivola2021good, JTVAE, lambo, lolbo, r2-cite9-lsbo-nfbo, r2-cite10-lsbo-cobo}, resulting in an algorithm known as \textit{latent space BO}.
Naturally, for problems with structured search spaces, we apply \ourmethod{} to the latent space of a pre-trained VAE on which we perform regular end-to-end updates with the surrogate model during optimization \citep{lolbo}.

\vspace{-1.5ex}
\paragraph{Plots.} 
In \cref{fig:main}, we plot the best coverage score $c$ \cref{eq:coverage_score} obtained by the $K$ best covering solutions found so far after a certain number of function evaluations on each task. 
Since BO baseline methods are not designed to optimize coverage directly, instead aiming to find a single solution or a set of solutions to optimize the $T$ objectives, we plot the best coverage score $c$ obtained by the $K$ best covering solutions found by the method. All plots show mean coverage scores averaged over $20$ replications of each method, and show standard errors. 

\vspace{-1.5ex}
\paragraph{Baselines.} 
In all plots, we compare \ourmethod{} against \cluso{}, \turbo{}, \robot{}, and \morbo{} \citep{cluso, turbo, robot, morbo}. 
While \morbo{} is not designed to optimize coverage, the MOBO setting is closely related to our problem setting, and \morbo{} represents a strong baseline among existing MOBO techniques.
We also compare to sampling uniformly at random in the search space. For each method, we compare to the performance of the best set of $K$ solutions found. 

\vspace{-1.5ex}
\paragraph{Extending baselines to the structured BO setting.} 
In the case of the two structured optimization tasks (molecule and peptide design), we apply \lolbo{} and \robot{} as described by design for these problem settings. For \morbo{} and \cluso{}, we take the straightforward approach of applying the methods directly in the continuous latent space of the VAE model without other adaptations. The same pre-trained VAE model is used across methods compared. 

\begin{table}[!ht]
\vskip -0.2in
\caption{\textbf{Best set of $K=4$ peptides found by one run of \ourmethod{} for the ``template free" peptide design task.}
We provide the APEX model's predicted MIC for each sequence on each of the $11$ target pathogenic bacteria.
The target pathogenic bacteria B1$, \ldots,$ B11 are listed in \cref{tab:bacteria}. (-) and (+) indicate Gram negative and Gram positive pathogenic bacteria respectively. 
The best/lowest MIC achieved for each pathogenic bacteria is in bold in each column. See row ``TF1" in \cref{fig:lab-coverage-only}, and \cref{fig:lab-result-full} for \textit{in vitro} MICs for this set of $K=4$ peptides. See \cref{tab:template-constrained-result} for an analogous result on the ``template constrained" peptide design task. 
}
\label{tab:template-free-result}
\centering
\vspace{1ex}
\resizebox{\columnwidth}{!}{
    \begin{tabular}{lccccccccccc}
        \toprule
        Peptide Amino Acid Sequence & B1(-) & B2(-) & B3(-) & B4(-) & B5(-) & B6(-) & B7(-) & B8(+) & B9(+) & B10(+) & B11(+) \\
        \midrule
        \texttt{KKKKLKLKKLKKLLKLLKRL}  & 1.017 & 1.040 & 1.893 & \bf{0.999} & 8.613 & \bf{0.966} & \bf{1.039} & 65.999 & 38.361 & 338.692 & 1.393\\
        \texttt{IFHLKILIKILRLL} & 0.999 & 15.565 & 1.860 & 1.952 & 404.254 & 486.860 & 406.034 & \bf{1.233} & \bf{1.318} & \bf{7.359} & \bf{0.981} \\
        \texttt{SKKIKLLGLALKLLKLKLKL}  &  2.654 & 3.268 & 3.113 & 4.854 & \bf{4.923} & 12.967 & 14.610 & 22.631 & 29.685 & 254.306 & 3.947\\
        \texttt{KKKKLKLKKLKRLLKLKLRL}  &  \bf{0.939} & \bf{0.906} & \bf{1.124} & 1.310 & 10.909 & 1.384 & 1.711 & 12.776 & 32.884 & 434.193 & 1.037\\
        \bottomrule
    \end{tabular}
}
    \vskip -0.1in
\end{table}

\vspace{-1.5ex}
\paragraph{Extending baselines to the coverage optimization setting.} 
As \turbo{}, \lolbo{}, and \robot{} target single-objective optimization, we conduct $T$ independent runs to optimize each of the $T$ objectives for the multi-objective task. 
We use all solutions gathered from the $T$ runs to compute the best covering set of $K$ solutions found by each method. 
Unlike \lolbo{} and \turbo{}, which seek a single best solution for a given objective, a single run of \robot{} seeks a set of $M$ solutions that are pairwise diverse. We run \robot{} with $M=K$ so that each independent run of \robot{} seeks $M=K$ diverse solutions. The aggregate result of the $T$ independent runs for \robot{} is thus $K$ diverse solutions for each of the $T$ objectives. We compare to the best covering $K$ solutions from among those $K*T$ solutions. See \cref{sec: robot-diversity-hypers} for more details on the diversity constraints used by \robot{} and the associated hyperparameters. 
\morbo{} is a multi-objective optimization method and can thus be applied directly to each multi-objective optimization task.
For each run of \morbo{}, we compare to the best covering set of $K$ solutions found among all solutions proposed by the run. 
For \cluso{}, no extension is needed as this method is designed for coverage optimization. We run \cluso{} with the same $K$ used by \ourmethod{}, and all other hyperparameters set to \cluso{} defaults. 

As far as we are aware, \cluso{} is the only existing method that tries to solve the black-box coverage optimization problem directly. 
We note that other adapted baselines are included not to criticize, but to underscore the importance of explicitly addressing this problem class. 

\vspace{-1.5ex}
\paragraph{T individually optimized solutions baseline.} 
We also compare to a brute-force method involving $T$ separate single-objective optimizations for each of the $T$ objectives, using \turbo{} for each run or \lolbo{} for molecule and peptide design. This is not an alternative for finding $K<T$ covering solutions, but instead identifies $T$ solutions, one per objective, approximating a ceiling we can achieve on performance without the limit of $K < T$ solutions. Approaching the performance of this baseline implies that we can find $K$ solutions that do nearly as well as if we were allowed $T$ solutions instead.

\vspace{-1ex}
\subsection{Tasks}
\label{sec: tasks}

\vspace{-1ex}
\paragraph{Peptide design.}
In the peptide design task, we explore amino acid sequences to minimize the MIC (minimum inhibitory concentration, measured in $\si{\micro\mole\per\liter}$) for each of $T=11$ target drug resistant strains or otherwise challenging to kill bacteria (B1-B7 Gram negative, B8-B11 Gram positive). \cref{tab:bacteria} lists our target bacteria in this study. Briefly, MIC indicates the concentration of peptide needed to inhibit bacterial growth (see \citealt{kowalska2021minimum}).
We evaluate MIC for a given peptide sequence and bacteria using the APEX 1.1 model proposed by \citet{apex1}. To frame the problem as maximization, we optimize $-\text{MIC}$. We seek $K=4$ peptides that together form a potent set of antibiotics for all $T=11$ bacteria. To enable optimization over peptides, we use the VAE model pre-trained on $4.5$ million amino acid sequences from \citet{apexgo} to map the peptide sequence search space to a continuous $256$ dimensional space.

\vspace{-1.5ex}
\paragraph{Template free vs template constrained peptide design.}
We evaluate \ourmethod{} on two variations of the peptide design task: ``template free" (TF) and ``template constrained" (TC). For TF, we allow the optimizer to propose any sequence of amino acids. For TC, we add a constraint that any sequence proposed by the optimizer must have a minimum of $75$ percent sequence similarity to at least one of the $10$ template amino acid sequences in \cref{tab:templates}. These $10$ templates were mined from extinct organisms and selected by \citet{apex1}. The motivation of the template constrained task is to design peptides specifically likely to evade antibiotic resistance by producing ``extinct-like'' peptides that bacteria have not encountered in nature in thousands of years. For \ourmethod{} and all BO baselines (\turbo{}, \lolbo{}, \robot{}, and \morbo{}), we handle the optimization constraint by adapting techniques from \texttt{SCBO} \citep{scbo}. For \cluso{} and random sampling, we handle the constraint with rejection sampling. 

\vspace{-1.5ex}
\paragraph{Ranolazine MPO molecule design.}
Ranolazine is a drug used to treat chest pain. The original \texttt{Ranolazine MPO} task from the Guacamol benchmark suite of molecular design tasks \citep{GuacaMol} aims to design an alternative to this drug: a molecule with a high fingerprint similarity to Ranolazine that includes fluorine.
We extend this task to the multi-objective optimization setting by searching for alternatives to Ranolazine that include $T=6$ reactive nonmetal elements not found in Ranolazine: fluorine, chlorine, bromine, selenium, sulfur, and phosphorus. We aim to cover the $T=6$ objectives with $K=3$ molecules. 
We use the SELFIES-VAE introduced by \citet{lolbo} to map the molecular space to a continuous $256$ dimensional space.

\begin{figure}[!ht]
\vspace{-2ex}
\begin{center}
\centerline{\includegraphics[width=0.6\columnwidth]{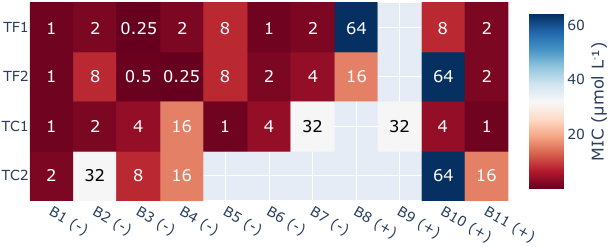}}
\caption{\textit{In vitro} results for the two best ``template free" (TF1, TF2) and two best ``template constrained" (TC1, TC2) runs of \ourmethod{} for the peptide design task. 
Columns are the best/lowest \textit{in vitro} MIC among the $K=4$ peptides found by \ourmethod{} for each target pathogenic bacteria B1$, \ldots,$ B11 listed in \cref{tab:bacteria}. (-) and (+) indicate Gram negative and Gram positive respectively. TF1 and TC1 correspond to the single runs of \ourmethod{} shown in \cref{tab:template-free-result} and \cref{tab:template-constrained-result} respectively. Methods used to obtain \textit{in vitro} MICs are provided in \cref{sec:marcelo-lab-methods}.
} 
\label{fig:lab-coverage-only}
\end{center}
\vskip -0.2in
\end{figure}
\begin{figure}[!ht]
\begin{center}
\centerline{\includegraphics[width=0.8\columnwidth]{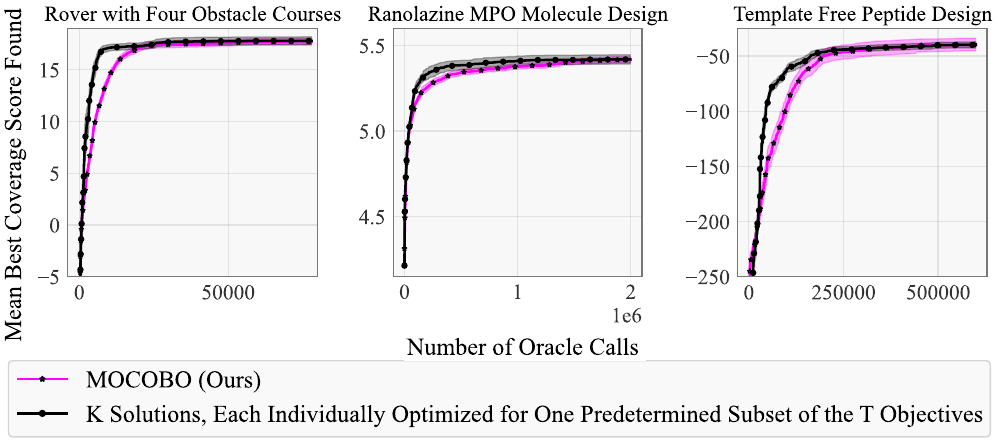}}
\caption{Ablation study comparing \ourmethod{} to optimization performance where a known ``good'' partitioning of the $T$ objectives into $K$ subsets is available in advance. We individually optimize $K$ solutions, one for each partition.
} 
\label{fig:ab-partition}
\end{center}
\vskip -0.3in
\end{figure}
\vspace{-1.5ex}
\paragraph{Rover.} 
The rover optimization task introduced by \citet{ebo} consists of finding a $60$-dimensional policy that allows a rover to move along some trajectory while avoiding a fixed set of obstacles. To frame this as a multi-objective optimization task, we design $T$ unique obstacle courses for the rover to navigate. The obstacle courses are designed such that no single policy can successfully navigate all courses. We seek $K < T$ policies so that at least one policy enables the rover to avoid obstacles in each course. We evaluate on three unique instances with varying numbers of obstacle courses. For the instances of this task with $T=4$ and $T=8$, we seek to cover the objectives with $K=2$ solutions. For the $T=12$ instance, we seek to cover the objectives with $K=4$ solutions. 
\vspace{-1.5ex}
\paragraph{Image tone mapping.}
In high dynamic range (HDR) images, some pixels (often associated with light sources) can dominate overall contrast, requiring adjustments to reveal detail in low-contrast areas, a problem known as \textit{tone mapping}~\citep[Section 6]{reinhard2005high}. Tone mapping algorithms involve various tunable parameters, resulting in a high-dimensional optimization problem of subjectively perceived quality. We seek a covering set of $K=4$ solutions to optimize a set of $T=7$ image aesthetic (IAA) and quality (IQA) assessment metrics from the \texttt{pyiqa} library~\citep{chaofeng2022iqapytorch} (see metrics listed in \cref{tab:img-metrics}). 
Our practical goal is that, while we do not know \textit{a priori} which metric is best for a particular image, covering all metrics may result in at least one high quality image.
We optimize over the $13$-dimensional parameter space of an established tone mapping pipeline to tone-map the ``Stanford Memorial Church"~\cite{1997recovering} and ``desk lamp"~\cite{cadik2008phd} benchmark images.
See \cref{sec:image-task-more-detials} for details.

\vspace{-1.5ex}
\subsection{Optimization Results}
\vspace{-1ex}
\label{sec: results}
In \cref{fig:main}, we provide optimization results comparing \ourmethod{} to the baselines discussed above on all tasks. 
The results show that \ourmethod{} finds sets of $K$ solutions that achieve higher coverage scores across tasks. The ``T Individually Optimized Solutions" baseline appears as a horizontal dotted line in all plots of \cref{fig:main}, representing the average coverage score of $T$ individually optimized solutions, serving as an approximation of the best possible performance \textit{without} the constraint of a limited $K < T$ solution set.
Results in \cref{fig:main} demonstrate that the smaller set of $K<T$ solutions identified by \ourmethod{} nearly equals the performance of the complete set of $T$ individually optimized solutions.
This result depends on using domain knowledge to choose $K$ large enough to achieve it. Results in \cref{fig:main} show that matching the performance of $T$ optimized solutions is possible with some values of $K \ll T$.

Although \cluso{} exhibits substantially lower performance than other baselines across the tasks we consider\textemdash{}including baselines not explicitly designed for the coverage problem\textemdash{}this should not be viewed as a criticism of the method. Rather, it reflects that \cluso{} was developed for low-dimensional black-box problems and does not scale effectively to the high-dimensional and structured optimization settings we study, further motivating the need for our scalable \ourmethod{} approach.

\vspace{-1.5ex}
\paragraph{Peptide design results.}
In \cref{tab:template-free-result}, we provide the $K=4$ peptides found by one run of \ourmethod{} for the template free (TF) peptide design task. 
For each peptide, we provide the the APEX 1.1 model's predicted MIC for each of the $11$ target bacteria.
MIC values $\leq \SI{16}{\micro\mole\per\liter}$ are considered to be ``highly active" against the target pathogenic bacteria \citep{apex1}.
We highlight in \cref{tab:template-free-result} the comparison between the second peptide with the other three peptides. 
B1-B7 are Gram negative (GN) bacteria, while B8-B11 are Gram positive (GP). Peptide 2 specialized to the GP bacteria (B8-B11 scores predicted highly active compared to the other peptides) at the expense of broad spectrum activity for the GN bacteria (B5, B6, B7 predicted inactive). By specializing its solutions to GN and GP bacteria separately, \ourmethod{} achieves low MIC across the target bacteria with only $K=4$ peptides. A similar table of results is also provided for the template constrained (TC) peptide design task (see \cref{tab:template-constrained-result}). 

\begin{figure*}[!ht]
\begin{center}
\centerline{\includegraphics[width=0.7\columnwidth]{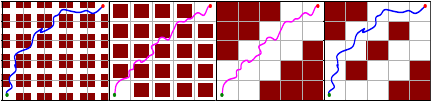}}
\caption{$T=4$ varied obstacle courses as multiple objectives for the rover task and $K=2$ covering trajectories (blue, magenta) found by \ourmethod{}. The first trajectory (magenta) successfully navigates obstacle courses 2 and 3 (second and third panels from the left). The second trajectory (blue) navigates obstacle courses 1 and 4 (first and fourth panels from the left).
} 
\label{fig:rover4-only}
\end{center}
\vskip -0.3in
\end{figure*}
\begin{figure*}[!ht]
\begin{center}
\centerline{\includegraphics[width=0.8\columnwidth]{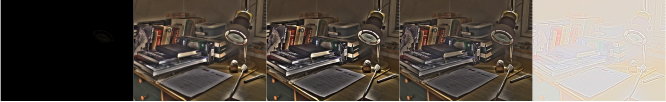}}
\caption{\textbf{(Leftmost Panel)} Initial hdr desk image for the desk variation of the ``image tone mapping" task. \textbf{(Four Rightmost Panels)} Images obtained by transforming the hdr desk image using the best covering set of $K=4$ solutions found by a single run of \ourmethod{}. 
} 
\label{fig:desk-only}
\end{center}
\vskip -0.3in
\end{figure*}

In \cref{fig:lab-coverage-only}, we provide \textit{in vitro} results for the two best TF and two best TC runs of \ourmethod{} for the peptide design task.
Here, ``best" means runs that achieved highest coverage scores according to the APEX 1.1 model. 
For each run, \cref{fig:lab-coverage-only} provides the best/lowest \textit{in vitro} MIC among the $K=4$ peptides found by \ourmethod{} for each target bacteria. 
Results demonstrate that solutions found by \ourmethod{} optimizing against the \textit{in-silico} APEX 1.1 model achieve good coverage of the $11$ target pathogenic bacteria \textit{in vitro}. 
In \cref{fig:lab-result-full}, we provide the full set of \textit{in vitro} results for each these runs of \ourmethod{}, with MIC values obtained by each of the $K=4$ peptides found for each target bacteria.
Methods used to obtain \textit{in vitro} MICs are provided in \cref{sec:marcelo-lab-methods}. 

\vspace{-1.5ex}
\paragraph{Molecule design results.}
In \cref{tab:rano-result}, we provide results for the $K=3$ molecules found by one run of \ourmethod{} for the molecule design task. Each of the $6$ target elements is successfully present in one of the $K=3$ molecules in the best covering set. These three molecules therefore effectively cover the $T=6$ objectives, with each objective having a max score $>0.9$.

\vspace{-1.5ex}
\paragraph{Rover results.}
In \cref{fig:rover4-only}, we depict the $4$ varied obstacle courses used for the $T=4$ variation of the rover task and $K=2$ covering trajectories found by a single run of \ourmethod{}.
In \cref{fig:rover8} and \ref{fig:rover12}, we provide analogous figures for the $T=8$ and $T=12$ variations. In each example, the \ourmethod{} optimized set of trajectories covers all obstacle courses such that all obstacles are avoided. 

\vspace{-1.5ex}
\paragraph{Image tone mapping results.}
In \cref{fig:desk-only} and \cref{fig:church_example}, we provide the original HDR images and the images obtained by the $K=4$ solutions found by a single run of \ourmethod{} for the church and desk test images, respectively. 
In both variations, three of the four solutions result in adequately tone-mapped images, while one results in a poor-quality image (see the middle-right church image and the rightmost desk image). 
The poor quality images were each generated by the one solution that \ourmethod{} used to cover metrics $5$ and $6$ (see metric ID numbers in \cref{tab:img-metrics}). 
This indicates that metrics $5$ and $6$ are poor indicators of true image quality in this setting. 
This result highlights a useful application of \ourmethod{}: By dedicating one of the $K$ solutions to maximizing the misleading metrics ($5$ and $6$), the other three solutions to can focus on the remaining subjectively better metrics.

\vspace{-1.5ex}
\subsection{Ablation study}
\vspace{-1ex}
\label{sec: ablation}
A challenging aspect of our problem setting is that we do not know \textit{a priori} the best way(s) to divide the $T$ objectives into $K$ subsets such that good coverage can be obtained. In this section, we seek to answer the question: \textit{what is the efficiency lost by \ourmethod{} due to not knowing an efficient partitioning of the objectives in advance?} To construct a proxy ``efficient'' partition to measure this, we first run \ourmethod{} to completion on a task, and then consider optimization as if we had known the partitioning of objectives found at the end in advance.

With such an ``oracle'' partitioning in hand ahead of time, we can efficiently optimize by running $K$ independent optimization runs: one to find a single solution for each subset in the fixed, given partition. With a fixed partition, the coverage score in \eqref{eq:coverage_score} reduces to optimizing the sum of objectives in each relevant subset independently. In \cref{fig:ab-partition}, we compare this strategy directly to \ourmethod{} on three optimization tasks. 
We plot the best coverage score obtained by $K$ independent runs of \turbo{}/\lolbo{} on the fixed oracle $K$-partitioning of the objectives by combining their solutions into a single set. \ourmethod{} achieves the same average best coverage scores with minimal loss in optimization efficiency \textit{despite having to discover an efficient partitioning during optimization.} 
See \cref{sec:ablation2} for additional ablation studies.

\vspace{-1ex}
\section{Related Works}\label{sec: related_works}
\vspace{-1ex}

The coverage optimization problem setting considered in this paper has not been previously explored in the Bayesian optimization literature. 
However, similar problem settings~\citep{icml25-r1-1,icml25-r1-2,icml25-r1-3} have been studied in the context of \textit{gradient}-based multi-objective optimization.
For instance, \citet{icml25-r1-1} demonstrated that the coverage problem generalizes various clustering methods, including k-means clustering, and proposed an algorithm that extends k-means++~\citep{kmeansplus} and Lloyd’s algorithm~\citep{lloyd-algo} to efficiently solve the coverage problem using gradient descent. 
\citet{icml25-r1-2} introduced Many-objective Multi-solution Transport (MosT), a novel framework that scales multi-objective gradient-based optimization to a large number of objectives. 
MosT uses an optimal transport to establish a balanced assignment between objectives and solutions, allowing each solution to focus on a specific subset of objectives and ensuring that the collective set of solutions covers all objectives. 
\citet{icml25-r1-3} introduced a novel Tchebycheff set (TCH-Set) scalarization approach for covering multiple conflicting objectives using a small set of collaborative solutions. 
They further propose a smooth Tchebycheff set (STCH-Set) scalarization method to handle non-smoothness in TCH-Set scalarization. 
The reliance on gradients makes these previous approaches inapplicable to the black-box setting.

For the black-box setting, \citet{cluso} was the first to consider the coverage problem.
They propose an analogous formalization of the coverage problem in the context of multi-objective black-box optimization, along with a clustering-based swarm optimization algorithm (\cluso{}) designed to solve it.
In \cref{sec:experiments}, we provide empirical results demonstrating that our Bayesian optimization approach (\ourmethod{}) consistently outperforms \cluso{} by a large margin across all tasks considered. 

Lastly, we note that \citet{r1-cite1} proposed a method with the same name as the ECI we defined in \cref{sec: methods}, which was later extended in~\citep{r1-cite2}.
The resemblance is only in the name, as their work addresses a fundamentally different notion of coverage.
Mainly, they attempt to cover the \textit{input space} by identifying a diverse set of feasible solutions that span a subspace defined by known threshold constraints on each objective.
In contrast, we attempt to ``cover'' the different \textit{objectives}.

\vspace{-1ex}
\section{Conclusions}
By bridging the gap between traditional multi-objective Bayesian optimization (BO) and practical coverage requirements, \ourmethod{} offers an effective approach to tackle critical problems in drug design and beyond. 
This framework extends the reach of Bayesian optimization into new domains, providing a robust solution to the challenges posed by extreme trade-offs and the need for specialized, collaborative solutions. See \cref{sec:limits} for additional discussion and limitations of \ourmethod{}.

\newpage
\begin{ack}
N. Maus was supported by the National Science Foundation Graduate Research Fellowship; 
K. Kim was supported by a gift from AWS AI to Penn Engineering's ASSET Center for Trustworthy AI;  
J. R. Gardner was supported by NSF awards IIS-2145644 and DBI-2400135;
C. de la Fuente-Nunez was supported by NIH grant R35GM138201 and Defense Threat Reduction Agency grants HDTRA11810041, HDTRA1-21-1-0014, and HDTRA1-23-1-0001. 
\end{ack}

\bibliographystyle{unsrtnat}
\bibliography{references}

\begin{thebibliography}{85}
\providecommand{\natexlab}[1]{#1}
\providecommand{\url}[1]{\texttt{#1}}
\expandafter\ifx\csname urlstyle\endcsname\relax
  \providecommand{\doi}[1]{doi: #1}\else
  \providecommand{\doi}{doi: \begingroup \urlstyle{rm}\Url}\fi

\bibitem[Jones et~al.(1998)Jones, Schonlau, and Welch]{jones1998efficient}
Donald~R. Jones, Matthias Schonlau, and William~J. Welch.
\newblock Efficient global optimization of expensive black-box functions.
\newblock \emph{Journal of Global Optimization}, 13\penalty0 (4):\penalty0 455--492, 1998.

\bibitem[Shahriari et~al.(2015)Shahriari, Swersky, Wang, Adams, and De~Freitas]{shahriari2015taking}
Bobak Shahriari, Kevin Swersky, Ziyu Wang, Ryan~P Adams, and Nando De~Freitas.
\newblock Taking the human out of the loop: {A} review of {Bayesian} optimization.
\newblock \emph{Proceedings of the IEEE}, 104\penalty0 (1):\penalty0 148--175, 2015.

\bibitem[Garnett(2023)]{garnett2023bayesian}
Roman Garnett.
\newblock \emph{Bayesian Optimization}.
\newblock {Cambridge University Press}, 2023.

\bibitem[Rasmussen and Williams(2005)]{rasmussen2005gaussian}
Carl~Edward Rasmussen and Christopher K.~I. Williams.
\newblock \emph{Gaussian Processes for Machine Learning}.
\newblock {The MIT Press}, November 2005.

\bibitem[Snoek et~al.(2012)Snoek, Larochelle, and Adams]{snoek2012practical}
Jasper Snoek, Hugo Larochelle, and Ryan~P Adams.
\newblock Practical {{Bayesian}} optimization of machine learning algorithms.
\newblock In \emph{Advances in Neural Information Processing Systems}, volume~25, pages 2951--2959. {Curran Associates, Inc.}, 2012.

\bibitem[Turner et~al.(2021)Turner, Eriksson, McCourt, Kiili, Laaksonen, Xu, and Guyon]{turner2021bayesian}
Ryan Turner, David Eriksson, Michael McCourt, Juha Kiili, Eero Laaksonen, Zhen Xu, and Isabelle Guyon.
\newblock {Bayesian} optimization is superior to random search for machine learning hyperparameter tuning: Analysis of the black-box optimization challenge 2020.
\newblock In \emph{Proceedings of the NeurIPS 2020 Competition and Demonstration Track}, volume 133 of \emph{PMLR}, pages 3--26, 2021.

\bibitem[Letham et~al.(2019)Letham, Karrer, Ottoni, and Bakshy]{letham2019noisyei}
Benjamin Letham, Brian Karrer, Guilherme Ottoni, and Eytan Bakshy.
\newblock Constrained {B}ayesian optimization with noisy experiments.
\newblock \emph{{Bayesian} Analysis}, 14\penalty0 (2):\penalty0 495--519, 2019.

\bibitem[Hern{\'a}ndez-Lobato et~al.(2017)Hern{\'a}ndez-Lobato, Requeima, Pyzer-Knapp, and Aspuru-Guzik]{hernandez2017parallel}
Jos{\'e}~Miguel Hern{\'a}ndez-Lobato, James Requeima, Edward~O Pyzer-Knapp, and Al{\'a}n Aspuru-Guzik.
\newblock Parallel and distributed {Thompson} sampling for large-scale accelerated exploration of chemical space.
\newblock In \emph{Proceedings of the International Conference on Machine Learning}, volume~70 of \emph{PMLR}, pages 1470--1479. JMLR, 2017.

\bibitem[Negoescu et~al.(2011)Negoescu, Frazier, and Powell]{negoescu2011knowledge}
Diana~M Negoescu, Peter~I Frazier, and Warren~B Powell.
\newblock The knowledge-gradient algorithm for sequencing experiments in drug discovery.
\newblock \emph{INFORMS Journal on Computing}, 23\penalty0 (3):\penalty0 346--363, 2011.

\bibitem[Hernández-Lobato et~al.(2016)Hernández-Lobato, Hernández-Lobato, Shah, and Adams]{mobo1}
Daniel Hernández-Lobato, José~Miguel Hernández-Lobato, Amar Shah, and Ryan~P. Adams.
\newblock Predictive entropy search for multi-objective {{Bayesian}} optimization.
\newblock In \emph{Proceedings of the International Conference on Machine Learning,}, volume~48 of \emph{PMLR}, pages 1492--1501. JMLR, 2016.

\bibitem[Belakaria et~al.(2019)Belakaria, Deshwal, and Doppa]{mobo2}
Syrine Belakaria, Aryan Deshwal, and Janardhan~Rao Doppa.
\newblock Max-value entropy search for multi-objective {{Bayesian}} optimization.
\newblock In \emph{Advances in Neural Information Processing Systems}, volume~32. Curran Associates, Inc., 2019.

\bibitem[Turchetta et~al.(2020)Turchetta, Krause, and Trimpe]{mobo3}
Matteo Turchetta, Andreas Krause, and Sebastian Trimpe.
\newblock Robust model-free reinforcement learning with multi-objective {{Bayesian}} optimization.
\newblock In \emph{Proceedings of the IEEE International Conference on Robotics and Automation}. IEEE, 2020.

\bibitem[Konakovic~Lukovic et~al.(2020)Konakovic~Lukovic, Tian, and Matusik]{dgemo}
Mina Konakovic~Lukovic, Yunsheng Tian, and Wojciech Matusik.
\newblock Diversity-guided multi-objective {{Bayesian}} optimization with batch evaluations.
\newblock In \emph{Advances in Neural Information Processing Systems}, volume~33, pages 17708--17720. Curran Associates, Inc., 2020.

\bibitem[Daulton et~al.(2021)Daulton, Eriksson, Balandat, and Bakshy]{morbo}
Samuel Daulton, David Eriksson, Maximilian Balandat, and Eytan Bakshy.
\newblock Multi-objective {{Bayesian}} optimization over high-dimensional search spaces.
\newblock In \emph{Proceedings of the Conference on Uncertainty in Artificial Intelligence}, volume 180 of \emph{PMLR}, pages 507--517. JMLR, 2021.

\bibitem[Stanton et~al.(2022)Stanton, Maddox, Gruver, Maffettone, Delaney, Greenside, and Wilson]{lambo}
Samuel Stanton, Wesley Maddox, Nate Gruver, Phillip Maffettone, Emily Delaney, Peyton Greenside, and Andrew~Gordon Wilson.
\newblock Accelerating {{Bayesian}} optimization for biological sequence design with denoising autoencoders.
\newblock In \emph{Proceedings of the International Conference on Machine Learning}, volume 162 of \emph{PMLR}, pages 20459--20478. JMLR, 2022.

\bibitem[Belakaria et~al.(2020)Belakaria, Deshwal, Jayakodi, and Doppa]{mobo4-usemo}
Syrine Belakaria, Aryan Deshwal, Nitthilan~Kannappan Jayakodi, and Janardhan~Rao Doppa.
\newblock Uncertainty-aware search framework for multi-objective {{Bayesian}} optimization.
\newblock In \emph{Proceedings of the AAAI Conference on Artificial Intelligence}, volume~30. AAAI Press, 2020.

\bibitem[Ding et~al.(2024)Ding, Chen, Wang, and Yin]{icml25-r1-1}
Lisang Ding, Ziang Chen, Xinshang Wang, and Wotao Yin.
\newblock Efficient algorithms for sum-of-minimum optimization.
\newblock In \emph{Proceedings of the International Conference on Machine Learning}, volume 235 of \emph{PMLR}, pages 10927--10959. JMLR, 2024.

\bibitem[Li et~al.(2025)Li, Li, Smith, Bilmes, and Zhou]{icml25-r1-2}
Ziyue Li, Tian Li, Virginia Smith, Jeff Bilmes, and Tianyi Zhou.
\newblock Many-objective multi-solution transport.
\newblock In \emph{Proceedings of International Conference on Learning Representations}, 2025.

\bibitem[Lin et~al.(2025)Lin, Liu, Zhang, Liu, Wang, and Zhang]{icml25-r1-3}
Xi~Lin, Yilu Liu, Xiaoyuan Zhang, Fei Liu, Zhenkun Wang, and Qingfu Zhang.
\newblock Few for many: Tchebycheff set scalarization for many-objective optimization.
\newblock In \emph{Proceedings of International Conference on Learning Representations}, 2025.

\bibitem[Liu et~al.(2024)Liu, Lu, Lin, and Zhang]{cluso}
Yilu Liu, Chengyu Lu, Xi~Lin, and Qingfu Zhang.
\newblock Many-objective cover problem: Discovering few solutions to cover many objectives.
\newblock In \emph{Parallel Problem Solving from Nature - {{PPSN XVIII}}}, volume 15151 of \emph{LNCS}, page 68–82. Springer-Verlag, 2024.

\bibitem[Eriksson et~al.(2019)Eriksson, Pearce, Gardner, Turner, and Poloczek]{turbo}
David Eriksson, Michael Pearce, Jacob Gardner, Ryan~D Turner, and Matthias Poloczek.
\newblock Scalable global optimization via local {{Bayesian}} optimization.
\newblock In \emph{Advances in Neural Information Processing Systems}, volume~32, pages 5496--5507. Curran Associates, Inc., 2019.

\bibitem[Gonz\'{a}lez-Duque et~al.(2024)Gonz\'{a}lez-Duque, Michael, Bartels, Zainchkovskyy, Hauberg, and Boomsma]{gonzalezduque2024survey}
Miguel Gonz\'{a}lez-Duque, Richard Michael, Simon Bartels, Yevgen Zainchkovskyy, S\o~ren Hauberg, and Wouter Boomsma.
\newblock A survey and benchmark of high-dimensional bayesian optimization of discrete sequences.
\newblock In \emph{Advances in Neural Information Processing Systems}, volume~37, pages 140478--140508. Curran Associates, Inc., 2024.

\bibitem[Torres et~al.(2024)Torres, Zeng, Wan, Maus, Gardner, and de~la Fuente-Nunez]{apexgo}
Marcelo D.~T. Torres, Yimeng Zeng, Fangping Wan, Natalie Maus, Jacob Gardner, and Cesar de~la Fuente-Nunez.
\newblock A generative artificial intelligence approach for antibiotic optimization.
\newblock bioRxiv Preprint 2024.11.27.625757, bioRxiv, 2024.

\bibitem[Wu et~al.(2025)Wu, Maus, Jha, Yang, Wales-McGrath, Jewell, Tangiyan, Choi, Gardner, and Barash]{wu20205generative}
Di~Wu, Natalie Maus, Anupama Jha, Kevin Yang, Benjamin~D Wales-McGrath, San Jewell, Anna Tangiyan, Peter Choi, Jacob~R Gardner, and Yoseph Barash.
\newblock Generative modeling for {{RNA}} splicing predictions and design.
\newblock March 2025.

\bibitem[Khan et~al.(2023)Khan, {Cowen-Rivers}, Grosnit, Deik, Robert, Greiff, Smorodina, Rawat, Akbar, Dreczkowski, Tutunov, {Bou-Ammar}, Wang, Storkey, and {Bou-Ammar}]{khan2023realworld}
Asif Khan, Alexander~I. {Cowen-Rivers}, Antoine Grosnit, Derrick-Goh-Xin Deik, Philippe~A. Robert, Victor Greiff, Eva Smorodina, Puneet Rawat, Rahmad Akbar, Kamil Dreczkowski, Rasul Tutunov, Dany {Bou-Ammar}, Jun Wang, Amos Storkey, and Haitham {Bou-Ammar}.
\newblock Toward real-world automated antibody design with combinatorial {{Bayesian}} optimization.
\newblock \emph{Cell Reports Methods}, 3\penalty0 (1), 2023.

\bibitem[Mockus(1982)]{mockus1982bayesian}
Jonas Mockus.
\newblock The {Bayesian} approach to global optimization.
\newblock In \emph{System Modeling and Optimization}, pages 473--481. Springer, 1982.

\bibitem[Nemhauser et~al.(1978)Nemhauser, Wolsey, and Fisher.]{greedysubmodular}
G.~L. Nemhauser, L.~A. Wolsey, and M.~L. Fisher.
\newblock An analysis of approximations for maximizing submodular set functions---{{I}}.
\newblock \emph{Mathematical Programming}, 14:\penalty0 265--294, 1978.

\bibitem[Maus et~al.(2023)Maus, Wu, Eriksson, and Gardner]{robot}
Natalie Maus, Kaiwen Wu, David Eriksson, and Jacob Gardner.
\newblock Discovering many diverse solutions with {{Bayesian}} optimization.
\newblock In \emph{Proceedings of {{the International Conference}} on {{Artificial Intelligence}} and {{Statistics}}}, volume 206 of \emph{{PMLR}}, pages 1779--1798. {JMLR}, April 2023.

\bibitem[Malkomes et~al.(2021)Malkomes, Cheng, Lee, and Mccourt]{r1-cite1}
Gustavo Malkomes, Bolong Cheng, Eric~H Lee, and Mike Mccourt.
\newblock Beyond the pareto efficient frontier: Constraint active search for multiobjective experimental design.
\newblock In \emph{Proceedings of the International Conference on Machine Learning}, volume 139 of \emph{PMLR}, pages 7423--7434. JMLR, 2021.

\bibitem[Feige(1998)]{set-cover-lnn}
Uriel Feige.
\newblock A threshold of $\mathrm{ln}\,n$ for approximating set cover.
\newblock \emph{Journal of the ACM}, 45\penalty0 (4):\penalty0 634--652, July 1998.

\bibitem[Wilson et~al.(2018)Wilson, Hutter, and Deisenroth]{wilson2018maximizing}
James Wilson, Frank Hutter, and Marc Deisenroth.
\newblock Maximizing acquisition functions for {{Bayesian}} optimization.
\newblock In \emph{Advances in {{Neural Information Processing Systems}}}, pages 9884--9895. {Curran Associates, Inc.}, 2018.

\bibitem[Wang et~al.(2020)Wang, Clark, Liu, and Frazier]{qei2}
Jialei Wang, Scott~C. Clark, Eric Liu, and Peter~I. Frazier.
\newblock Parallel {{Bayesian}} global optimization of expensive functions.
\newblock \emph{Operations Research}, 68\penalty0 (6):\penalty0 1850--1865, 2020.

\bibitem[Balandat et~al.(2020)Balandat, Karrer, Jiang, Daulton, Letham, Wilson, and Bakshy]{balandat2020botorch}
Maximilian Balandat, Brian Karrer, Daniel~R. Jiang, Samuel Daulton, Benjamin Letham, Andrew~Gordon Wilson, and Eytan Bakshy.
\newblock {BoTorch}: A framework for efficient {{Monte-Carlo Bayesian}} optimization.
\newblock In \emph{Advances in Neural Information Processing Systems}, volume~33, pages 21524--21538. Curran Associates, Inc., 2020.

\bibitem[Gardner et~al.(2018)Gardner, Pleiss, Weinberger, Bindel, and Wilson]{gardner2018gpytorch}
Jacob Gardner, Geoff Pleiss, Kilian~Q. Weinberger, David Bindel, and Andrew~G. Wilson.
\newblock {{GPyTorch}}: {{Blackbox}} matrix-matrix {{Gaussian}} process inference with {{GPU}} acceleration.
\newblock In \emph{Advances in {{Neural Information Processing Systems}}}, volume~31, pages 7576--7586. {Curran Associates, Inc.}, 2018.

\bibitem[Jankowiak et~al.(2020)Jankowiak, Pleiss, and Gardner]{PPGPR}
Martin Jankowiak, Geoff Pleiss, and Jacob~R. Gardner.
\newblock Parametric {{Gaussian}} process regressors.
\newblock In \emph{Proceedings of the International Conference on Machine Learning}, volume 119 of \emph{PMLR}, pages 4702--4712. JMLR, 2020.

\bibitem[Kingma and Welling(2014)]{kingma2013auto}
Diederik~P. Kingma and Max Welling.
\newblock Auto-encoding variational {{Bayes}}.
\newblock In \emph{Proceedings of the {{International Conference}} on {{Learning Representations}}}, {Banff, AB, Canada}, April 2014.

\bibitem[Rezende et~al.(2014)Rezende, Mohamed, and Wierstra]{rezende2014stochastic}
Danilo~Jimenez Rezende, Shakir Mohamed, and Daan Wierstra.
\newblock Stochastic backpropagation and approximate inference in deep generative models.
\newblock In \emph{Proceedings of the {{International Conference}} on {{Machine Learning}}}, volume~32 of \emph{{{PMLR}}}, pages 1278--1286. {JMLR}, June 2014.

\bibitem[Eissman et~al.(2018)Eissman, Levy, Shu, Bartzsch, and Ermon]{eissman2018bayesian}
Stephan Eissman, Daniel Levy, Rui Shu, Stefan Bartzsch, and Stefano Ermon.
\newblock Bayesian optimization and attribute adjustment.
\newblock In \emph{Proceedings of the Conference on Uncertainty in Artificial Intelligence}. AUAI Press, 2018.

\bibitem[Notin et~al.(2021)Notin, Hern\'{a}ndez-Lobato, and Gal]{r2-cite5-lsbo}
Pascal Notin, Jos\'{e}~Miguel Hern\'{a}ndez-Lobato, and Yarin Gal.
\newblock Improving black-box optimization in vae latent space using decoder uncertainty.
\newblock In \emph{Advances in Neural Information Processing Systems}, volume~34, pages 802--814. Curran Associates, Inc., 2021.

\bibitem[Lee et~al.(2023{\natexlab{a}})Lee, Chu, Kim, Ko, and Kim]{r2-cite10-lsbo-cobo}
Seunghun Lee, Jaewon Chu, Sihyeon Kim, Juyeon Ko, and Hyunwoo~J Kim.
\newblock Advancing bayesian optimization via learning correlated latent space.
\newblock In A.~Oh, T.~Naumann, A.~Globerson, K.~Saenko, M.~Hardt, and S.~Levine, editors, \emph{Advances in Neural Information Processing Systems}, volume~36, pages 48906--48917. Curran Associates, Inc., 2023{\natexlab{a}}.

\bibitem[Kusner et~al.(2017)Kusner, Paige, and Hern{\'a}ndez-Lobato]{r2-cite7-lsbo-gvae}
Matt~J Kusner, Brooks Paige, and Jos{\'e}~Miguel Hern{\'a}ndez-Lobato.
\newblock Grammar variational autoencoder.
\newblock In \emph{Proceedings of the International Conference on Machine Learning}, volume~70 of \emph{PMLR}, pages 1945--1954. PMLR, 2017.

\bibitem[Lu et~al.(2018)Lu, Gonzalez, Dai, and Lawrence]{r2-cite8-lsbo}
Xiaoyu Lu, Javier Gonzalez, Zhenwen Dai, and Neil Lawrence.
\newblock Structured variationally auto-encoded optimization.
\newblock In \emph{Proceedings of the International Conference on Machine Learning}, volume~80 of \emph{PMLR}, pages 3267--3275. JMLR, 10--15 Jul 2018.

\bibitem[Tripp et~al.(2020)Tripp, Daxberger, and Hern{\'{a}}ndez{-}Lobato]{Weighted_Retraining}
Austin Tripp, Erik~A. Daxberger, and Jos{\'{e}}~Miguel Hern{\'{a}}ndez{-}Lobato.
\newblock Sample-efficient optimization in the latent space of deep generative models via weighted retraining.
\newblock In \emph{Advances in Neural Information Processing Systems 33}. Curran Associates, Inc., 2020.

\bibitem[Grosnit et~al.(2021)Grosnit, Tutunov, Maraval, Griffiths, Cowen{-}Rivers, Yang, Zhu, Lyu, Chen, Wang, Peters, and Bou{-}Ammar]{Huawei}
Antoine Grosnit, Rasul Tutunov, Alexandre~Max Maraval, Ryan{-}Rhys Griffiths, Alexander~Imani Cowen{-}Rivers, Lin Yang, Lin Zhu, Wenlong Lyu, Zhitang Chen, Jun Wang, Jan Peters, and Haitham Bou{-}Ammar.
\newblock High-dimensional {Bayesian} optimisation with variational autoencoders and deep metric learning.
\newblock {{arXiv}} Preprint arXiv:2106.03609, arXiv, 2021.

\bibitem[Siivola et~al.(2021)Siivola, Paleyes, Gonz{\'a}lez, and Vehtari]{siivola2021good}
Eero Siivola, Andrei Paleyes, Javier Gonz{\'a}lez, and Aki Vehtari.
\newblock Good practices for {Bayesian} optimization of high dimensional structured spaces.
\newblock \emph{Applied AI Letters}, 2\penalty0 (2):\penalty0 e24, 2021.

\bibitem[Jin et~al.(2018)Jin, Barzilay, and Jaakkola]{JTVAE}
Wengong Jin, Regina Barzilay, and Tommi~S. Jaakkola.
\newblock Junction tree variational autoencoder for molecular graph generation.
\newblock In \emph{Proceedings of the International Conference on Machine Learning}, volume~80, pages 2323--2332. JMLR, 2018.

\bibitem[Maus et~al.(2022)Maus, Jones, Moore, Kusner, Bradshaw, and Gardner]{lolbo}
Natalie Maus, Haydn Jones, Juston Moore, Matt~J. Kusner, John Bradshaw, and Jacob Gardner.
\newblock Local latent space {{Bayesian}} optimization over structured inputs.
\newblock In \emph{Advances in Neural Information Processing Systems}, volume~35, pages 34505--34518, December 2022.

\bibitem[Lee et~al.(2025)Lee, Park, Chu, Yoon, and Kim]{r2-cite9-lsbo-nfbo}
Seunghun Lee, Jinyoung Park, Jaewon Chu, Minseo Yoon, and Hyunwoo~J. Kim.
\newblock Latent {{Bayesian}} optimization via autoregressive normalizing flows.
\newblock In \emph{Proceedings of the International Conference on Learning Representations}, 2025.

\bibitem[Kowalska-Krochmal and Dudek-Wicher(2021)]{kowalska2021minimum}
Beata Kowalska-Krochmal and Ruth Dudek-Wicher.
\newblock The minimum inhibitory concentration of antibiotics: Methods, interpretation, clinical relevance.
\newblock \emph{Pathogens}, 10\penalty0 (2):\penalty0 165, 2021.

\bibitem[Wan et~al.(2024)Wan, Torres, Peng, and de~la Fuente-Nunez]{apex1}
Fangping Wan, Marcelo D.~T. Torres, Jacqueline Peng, and Cesar de~la Fuente-Nunez.
\newblock Deep-learning-enabled antibiotic discovery through molecular de-extinction.
\newblock \emph{Nature Biomedical Engineering}, 8\penalty0 (7):\penalty0 854--871, Jul 2024.

\bibitem[Eriksson and Poloczek(2021)]{scbo}
David Eriksson and Matthias Poloczek.
\newblock Scalable constrained {Bayesian} optimization.
\newblock In \emph{Proceedings of the International Conference on Artificial Intelligence and Statistics}, volume 130 of \emph{PMLR}, pages 730--738. JMLR, 2021.

\bibitem[Brown et~al.(2019)Brown, Fiscato, Segler, and Vaucher]{GuacaMol}
Nathan Brown, Marco Fiscato, Marwin~H.S. Segler, and Alain~C. Vaucher.
\newblock Guacamol: Benchmarking models for de novo molecular design.
\newblock \emph{Journal of Chemical Information and Modeling}, 59\penalty0 (3):\penalty0 1096–1108, Mar 2019.

\bibitem[Wang et~al.(2018)Wang, Gehring, Kohli, and Jegelka]{ebo}
Zi~Wang, Clement Gehring, Pushmeet Kohli, and Stefanie Jegelka.
\newblock Batched large-scale {{Bayesian}} optimization in high-dimensional spaces.
\newblock In \emph{Proceedings of the {{International Conference}} on {{Artificial Intelligence}} and {{Statistics}}}, volume~84 of \emph{{{PMLR}}}, pages 745--754. {JMLR}, March 2018.

\bibitem[Reinhard et~al.(2005)Reinhard, Ward, Pattanaik, and Debevec]{reinhard2005high}
Erik Reinhard, Greg Ward, Sumanta Pattanaik, and Paul Debevec.
\newblock \emph{High Dynamic Range Imaging: Acquisition, Display, and Image-Based Lighting (The Morgan Kaufmann Series in Computer Graphics)}.
\newblock Morgan Kaufmann Publishers Inc., San Francisco, CA, USA, 2005.

\bibitem[Chen and Mo(2022)]{chaofeng2022iqapytorch}
Chaofeng Chen and Jiadi Mo.
\newblock {IQA-PyTorch}: Pytorch toolbox for image quality assessment.
\newblock [Online]. Available: \url{https://github.com/chaofengc/IQA-PyTorch}, 2022.

\bibitem[Debevec and Malik(1997)]{1997recovering}
Paul~E. Debevec and Jitendra Malik.
\newblock Recovering high dynamic range radiance maps from photographs.
\newblock In \emph{Proceedings of the Annual Conference on Computer Graphics and Interactive Techniques}, SIGGRAPH '97, page 369–378. ACM Press/Addison-Wesley Publishing Co., 1997.

\bibitem[{\v{C}}ad\'{i}k(2008)]{cadik2008phd}
Martin {\v{C}}ad\'{i}k.
\newblock \emph{Perceptually Based Image Quality Assessment and Image Transformations}.
\newblock Ph.{D}. thesis, Department of Computer Science and Engineering, Faculty of Electrical Engineering, Czech Technical University in Prague, January 2008.
\newblock URL \url{https://cadik.posvete.cz/diss/}.

\bibitem[Arthur and Vassilvitskii(2007)]{kmeansplus}
David Arthur and Sergei Vassilvitskii.
\newblock {{K-Means++}}: {{The}} advantages of careful seeding.
\newblock In \emph{Proceedings of the Annual ACM-SIAM Symposium on Discrete Algorithms}, SODA '07, page 1027–1035, USA, 2007. Society for Industrial and Applied Mathematics.

\bibitem[Lloyd(1982)]{lloyd-algo}
S.~Lloyd.
\newblock Least squares quantization in {{PCM}}.
\newblock \emph{IEEE Transactions on Information Theory}, 28\penalty0 (2):\penalty0 129--137, 1982.

\bibitem[Lee et~al.(2023{\natexlab{b}})Lee, Cheng, and McCourt]{r1-cite2}
Eric~Hans Lee, Bolong Cheng, and Michael McCourt.
\newblock Achieving diversity in objective space for sample-efficient search of multiobjective optimization problems.
\newblock In \emph{Proceedings of the Winter Simulation Conference}, WSC '22, page 3146–3157. IEEE, 2023{\natexlab{b}}.

\bibitem[Cesaro et~al.(2022)Cesaro, Torres, and {de la Fuente-Nunez}]{labmethods}
Angela Cesaro, Marcelo Der~Torossian Torres, and Cesar {de la Fuente-Nunez}.
\newblock Methods for the design and characterization of peptide antibiotics.
\newblock In Leslie~M. Hicks, editor, \emph{Antimicrobial Peptides}, volume 663 of \emph{Methods in Enzymology}, chapter~13, pages 303--326. Academic Press, 2022.

\bibitem[Weininger(1988)]{SMILES}
David Weininger.
\newblock {SMILES}, a chemical language and information system. 1. introduction to methodology and encoding rules.
\newblock \emph{Journal of Chemical Information and Computer Sciences}, 28\penalty0 (1):\penalty0 31--36, 1988.

\bibitem[Hvarfner et~al.(2024)Hvarfner, Hellsten, and Nardi]{carlbo}
Carl Hvarfner, Erik~Orm Hellsten, and Luigi Nardi.
\newblock Vanilla bayesian optimization performs great in high dimensions.
\newblock In \emph{Proceedings of the International Conference on Machine Learning}, volume 235 of \emph{PMLR}, pages 20793--20817. JMLR, 2024.

\bibitem[González-Duque et~al.(2024)González-Duque, Michael, Bartels, Zainchkovskyy, Hauberg, and Boomsma]{structured-bo-survey}
Miguel González-Duque, Richard Michael, Simon Bartels, Yevgen Zainchkovskyy, Søren Hauberg, and Wouter Boomsma.
\newblock A survey and benchmark of high-dimensional bayesian optimization of discrete sequences.
\newblock In \emph{Advances in Neural Information Processing Systems (Track on Database and Benchmarks)}, volume~37, pages 140478--140508. Curran Associates, Inc., 2024.

\bibitem[Urbina et~al.(2020)Urbina, Lentzos, Invernizzi, and Ekins]{dualuse}
Fabio Urbina, Filippa Lentzos, Cédric Invernizzi, and Sean Ekins.
\newblock Dual use of artificial-intelligence-powered drug discovery.
\newblock \emph{Nature Machine Intelligence}, 4:\penalty0 189--191, 2020.

\bibitem[Wilson et~al.(2016)Wilson, Hu, Salakhutdinov, and Xing]{dkl}
Andrew~G Wilson, Zhiting Hu, Russ~R Salakhutdinov, and Eric~P Xing.
\newblock Stochastic variational deep kernel learning.
\newblock In \emph{Advances in Neural Information Processing Systems}, volume~29, pages 2586--2594. Curran Associates, Inc., 2016.

\bibitem[Jiang et~al.(2018)Jiang, Gao, Liu, Cai, Zhang, and Liu]{shared-dkl}
Xinwei Jiang, Junbin Gao, Xiaobo Liu, Zhihua Cai, Dongmei Zhang, and Yuanxing Liu.
\newblock Shared deep kernel learning for dimensionality reduction.
\newblock In \emph{Advances in Knowledge Discovery and Data Mining}, pages 297--308, Cham, 2018. Springer International Publishing.

\bibitem[Patacchiola et~al.(2020)Patacchiola, Turner, Crowley, O\textquotesingle~Boyle, and Storkey]{dkt}
Massimiliano Patacchiola, Jack Turner, Elliot~J. Crowley, Michael O\textquotesingle~Boyle, and Amos~J Storkey.
\newblock Bayesian meta-learning for the few-shot setting via deep kernels.
\newblock In \emph{Advances in Neural Information Processing Systems}, volume~33, pages 16108--16118. Curran Associates, Inc., 2020.

\bibitem[Kingma and Ba(2015)]{adam-optimizer}
Diederik~P. Kingma and Jimmy Ba.
\newblock Adam: A method for stochastic optimization.
\newblock In \emph{Proceedings of the International Conference on Learning Representations}, 2015.

\bibitem[Talebi and Milanfar(2018)]{talebi2018nima}
Hossein Talebi and Peyman Milanfar.
\newblock {{NIMA}}: {{Neural}} image assessment.
\newblock \emph{IEEE Transactions on Image Processing}, 27\penalty0 (8):\penalty0 3998--4011, 2018.

\bibitem[Murray et~al.(2012)Murray, Marchesotti, and Perronnin]{murray2012ava}
Naila Murray, Luca Marchesotti, and Florent Perronnin.
\newblock {{AVA}}: A large-scale database for aesthetic visual analysis.
\newblock In \emph{Proceedings of the IEEE Conference on Computer Vision and Pattern Recognition}. IEEE, 2012.

\bibitem[Chen et~al.(2024)Chen, Mo, Hou, Wu, Liao, Sun, Yan, and Lin]{chen2024topiq}
Chaofeng Chen, Jiadi Mo, Jingwen Hou, Haoning Wu, Liang Liao, Wenxiu Sun, Qiong Yan, and Weisi Lin.
\newblock {{TOPIQ}}: {{A}} top-down approach from semantics to distortions for image quality assessment.
\newblock \emph{IEEE Transactions on Image Processing}, 33:\penalty0 2404--2418, 2024.

\bibitem[Schuhmann et~al.(2022)Schuhmann, Beaumont, Vencu, Gordon, Wightman, Cherti, Coombes, Katta, Mullis, Wortsman, Schramowski, Kundurthy, Crowson, Schmidt, Kaczmarczyk, and Jitsev]{schuhmann2022laion}
Christoph Schuhmann, Romain Beaumont, Richard Vencu, Cade Gordon, Ross Wightman, Mehdi Cherti, Theo Coombes, Aarush Katta, Clayton Mullis, Mitchell Wortsman, Patrick Schramowski, Srivatsa Kundurthy, Katherine Crowson, Ludwig Schmidt, Robert Kaczmarczyk, and Jenia Jitsev.
\newblock {{LAION-5B}}: An open large-scale dataset for training next generation image-text models.
\newblock In \emph{Advances in Neural Information Processing Systems}, volume~35, pages 25278--25294. Curran Associates, Inc., 2022.

\bibitem[Su et~al.(2020)Su, Yan, Zhu, Zhang, Ge, Sun, and Zhang]{su2020blindly}
Shaolin Su, Qingsen Yan, Yu~Zhu, Cheng Zhang, Xin Ge, Jinqiu Sun, and Yanning Zhang.
\newblock Blindly assess image quality in the wild guided by a self-adaptive hyper network.
\newblock In \emph{Proceedings of the IEEE/CVF Conference on Computer Vision and Pattern Recognition}, pages 3664--3673, 2020.

\bibitem[Golestaneh et~al.(2022)Golestaneh, Dadsetan, and Kitani]{golestaneh2022noreference}
S.~Alireza Golestaneh, Saba Dadsetan, and Kris~M. Kitani.
\newblock No-reference image quality assessment via transformers, relative ranking, and self-consistency.
\newblock In \emph{Proceedings of the IEEE/CVF Winter Conference on Applications of Computer Vision}, pages 1220--1230, January 2022.

\bibitem[Zhang et~al.(2023)Zhang, Zhai, Wei, Yang, and Ma]{zhang2023cvpr}
Weixia Zhang, Guangtao Zhai, Ying Wei, Xiaokang Yang, and Kede Ma.
\newblock Blind image quality assessment via vision-language correspondence: {{A}} multitask learning perspective.
\newblock In \emph{Proceedings of the IEEE/CVF Conference on Computer Vision and Pattern Recognition}, pages 14071--14081, June 2023.

\bibitem[Lischinski et~al.(2006)Lischinski, Farbman, Uyttendaele, and Szeliski]{lischinski2006interactive}
Dani Lischinski, Zeev Farbman, Matt Uyttendaele, and Richard Szeliski.
\newblock Interactive local adjustment of tonal values.
\newblock \emph{ACM Transactions on Graphics}, 25\penalty0 (3):\penalty0 646–653, July 2006.

\bibitem[Koyama et~al.(2017)Koyama, Sato, Sakamoto, and Igarashi]{koyama2017sequential}
Yuki Koyama, Issei Sato, Daisuke Sakamoto, and Takeo Igarashi.
\newblock Sequential line search for efficient visual design optimization by crowds.
\newblock \emph{ACM Transations on Graphics}, 36\penalty0 (4), July 2017.

\bibitem[Koyama et~al.(2020)Koyama, Sato, and Goto]{koyama2020sequential}
Yuki Koyama, Issei Sato, and Masataka Goto.
\newblock Sequential gallery for interactive visual design optimization.
\newblock \emph{ACM Transactions on Graphics}, 39\penalty0 (4), August 2020.

\bibitem[Tumblin and Turk(1999)]{tumblin1999lcis}
Jack Tumblin and Greg Turk.
\newblock {{LCIS}}: {{A}} boundary hierarchy for detail-preserving contrast reduction.
\newblock In \emph{Proceedings of the Annual Conference on Computer Graphics and Interactive Techniques}, page 83–90, USA, 1999. ACM Press/Addison-Wesley Publishing Co.

\bibitem[Durand and Dorsey(2002)]{durand2002fast}
Fr\'{e}do Durand and Julie Dorsey.
\newblock Fast bilateral filtering for the display of high-dynamic-range images.
\newblock In \emph{Proceedings of the 29th Annual Conference on Computer Graphics and Interactive Techniques}, SIGGRAPH '02, page 257–266, New York, NY, USA, 2002. Association for Computing Machinery.

\bibitem[Li et~al.(2005)Li, Sharan, and Adelson]{li2005compressing}
Yuanzhen Li, Lavanya Sharan, and Edward~H. Adelson.
\newblock Compressing and companding high dynamic range images with subband architectures.
\newblock \emph{ACM Transactions on Graphics}, 24\penalty0 (3):\penalty0 836–844, July 2005.

\bibitem[He et~al.(2013)He, Sun, and Tang]{he2013guided}
Kaiming He, Jian Sun, and Xiaoou Tang.
\newblock Guided image filtering.
\newblock \emph{IEEE Transactions on Pattern Analysis and Machine Intelligence}, 35\penalty0 (6):\penalty0 1397--1409, 2013.

\bibitem[Farbman et~al.(2008)Farbman, Fattal, Lischinski, and Szeliski]{farbman2008edge}
Zeev Farbman, Raanan Fattal, Dani Lischinski, and Richard Szeliski.
\newblock Edge-preserving decompositions for multi-scale tone and detail manipulation.
\newblock \emph{ACM Transactions on Graphics}, 27\penalty0 (3):\penalty0 1–10, August 2008.

\bibitem[Bradski(2000)]{Bradski2000opencv}
G.~Bradski.
\newblock {The OpenCV Library}.
\newblock \emph{Dr. Dobb's Journal of Software Tools}, 2000.

\end{thebibliography}

\newpage
\section*{NeurIPS Paper Checklist}

\begin{enumerate}

\item {\bf Claims}
    \item[] Question: Do the main claims made in the abstract and introduction accurately reflect the paper's contributions and scope?
    \item[] Answer: \answerYes{} % Replace by \answerYes{}, \answerNo{}, or \answerNA{}.
    \item[] Justification: All stated claims are backed-up with results in \cref{sec:experiments} and the stated focus/scope of the paper accurately reflects what is discussed throughout the rest of the paper. 
    \item[] Guidelines:
    \begin{itemize}
        \item The answer NA means that the abstract and introduction do not include the claims made in the paper.
        \item The abstract and/or introduction should clearly state the claims made, including the contributions made in the paper and important assumptions and limitations. A No or NA answer to this question will not be perceived well by the reviewers. 
        \item The claims made should match theoretical and experimental results, and reflect how much the results can be expected to generalize to other settings. 
        \item It is fine to include aspirational goals as motivation as long as it is clear that these goals are not attained by the paper. 
    \end{itemize}
\item {\bf Limitations}
    \item[] Question: Does the paper discuss the limitations of the work performed by the authors?
    \item[] Answer: \answerYes{} % Replace by \answerYes{}, \answerNo{}, or \answerNA{}.
    \item[] Justification: See \cref{sec:limits} for discussion of limitations.  
    \item[] Guidelines:
    \begin{itemize}
        \item The answer NA means that the paper has no limitation while the answer No means that the paper has limitations, but those are not discussed in the paper. 
        \item The authors are encouraged to create a separate "Limitations" section in their paper.
        \item The paper should point out any strong assumptions and how robust the results are to violations of these assumptions (e.g., independence assumptions, noiseless settings, model well-specification, asymptotic approximations only holding locally). The authors should reflect on how these assumptions might be violated in practice and what the implications would be.
        \item The authors should reflect on the scope of the claims made, e.g., if the approach was only tested on a few datasets or with a few runs. In general, empirical results often depend on implicit assumptions, which should be articulated.
        \item The authors should reflect on the factors that influence the performance of the approach. For example, a facial recognition algorithm may perform poorly when image resolution is low or images are taken in low lighting. Or a speech-to-text system might not be used reliably to provide closed captions for online lectures because it fails to handle technical jargon.
        \item The authors should discuss the computational efficiency of the proposed algorithms and how they scale with dataset size.
        \item If applicable, the authors should discuss possible limitations of their approach to address problems of privacy and fairness.
        \item While the authors might fear that complete honesty about limitations might be used by reviewers as grounds for rejection, a worse outcome might be that reviewers discover limitations that aren't acknowledged in the paper. The authors should use their best judgment and recognize that individual actions in favor of transparency play an important role in developing norms that preserve the integrity of the community. Reviewers will be specifically instructed to not penalize honesty concerning limitations.
    \end{itemize}
\item {\bf Theory assumptions and proofs}
    \item[] Question: For each theoretical result, does the paper provide the full set of assumptions and a complete (and correct) proof?
    \item[] Answer: \answerYes{} % Replace by \answerYes{}, \answerNo{}, or \answerNA{}.
    \item[] Justification: All theoretical claims can be found in \cref{sec:greedy}. All sets of assumptions and full proofs of each claim are provided in \cref{sec:greedy}, \cref{sec: nphard}, and \cref{sec: greedyproof}. As far as we are aware, each proof is complete and correct. 
    \item[] Guidelines:
    \begin{itemize}
        \item The answer NA means that the paper does not include theoretical results. 
        \item All the theorems, formulas, and proofs in the paper should be numbered and cross-referenced.
        \item All assumptions should be clearly stated or referenced in the statement of any theorems.
        \item The proofs can either appear in the main paper or the supplemental material, but if they appear in the supplemental material, the authors are encouraged to provide a short proof sketch to provide intuition. 
        \item Inversely, any informal proof provided in the core of the paper should be complemented by formal proofs provided in appendix or supplemental material.
        \item Theorems and Lemmas that the proof relies upon should be properly referenced. 
    \end{itemize}

    \item {\bf Experimental result reproducibility}
    \item[] Question: Does the paper fully disclose all the information needed to reproduce the main experimental results of the paper to the extent that it affects the main claims and/or conclusions of the paper (regardless of whether the code and data are provided or not)?
    \item[] Answer: \answerYes{} % Replace by \answerYes{}, \answerNo{}, or \answerNA{}.
    \item[] Justification: We provide detailed explanation of how our method works in \cref{sec: methods} and all additional required details to reproduce results in \cref{sec:experiments} and \cref{sec:detials}. 
    Additionally, we provide a link to a public GitHub repository containing the source code used to run our method and produce results provided in this paper.
    This GitHub repository contains organized code that will allow any reader to run \ourmethod{} on all tasks in this paper. Additionally, the README in the repository provides detailed instructions to make setting up the proper environment and running the code easy for users.
    \item[] Guidelines:
    \begin{itemize}
        \item The answer NA means that the paper does not include experiments.
        \item If the paper includes experiments, a No answer to this question will not be perceived well by the reviewers: Making the paper reproducible is important, regardless of whether the code and data are provided or not.
        \item If the contribution is a dataset and/or model, the authors should describe the steps taken to make their results reproducible or verifiable. 
        \item Depending on the contribution, reproducibility can be accomplished in various ways. For example, if the contribution is a novel architecture, describing the architecture fully might suffice, or if the contribution is a specific model and empirical evaluation, it may be necessary to either make it possible for others to replicate the model with the same dataset, or provide access to the model. In general. releasing code and data is often one good way to accomplish this, but reproducibility can also be provided via detailed instructions for how to replicate the results, access to a hosted model (e.g., in the case of a large language model), releasing of a model checkpoint, or other means that are appropriate to the research performed.
        \item While NeurIPS does not require releasing code, the conference does require all submissions to provide some reasonable avenue for reproducibility, which may depend on the nature of the contribution. For example
        \begin{enumerate}
            \item If the contribution is primarily a new algorithm, the paper should make it clear how to reproduce that algorithm.
            \item If the contribution is primarily a new model architecture, the paper should describe the architecture clearly and fully.
            \item If the contribution is a new model (e.g., a large language model), then there should either be a way to access this model for reproducing the results or a way to reproduce the model (e.g., with an open-source dataset or instructions for how to construct the dataset).
            \item We recognize that reproducibility may be tricky in some cases, in which case authors are welcome to describe the particular way they provide for reproducibility. In the case of closed-source models, it may be that access to the model is limited in some way (e.g., to registered users), but it should be possible for other researchers to have some path to reproducing or verifying the results.
        \end{enumerate}
    \end{itemize}

\item {\bf Open access to data and code}
    \item[] Question: Does the paper provide open access to the data and code, with sufficient instructions to faithfully reproduce the main experimental results, as described in supplemental material?
    \item[] Answer: \answerYes{} % Replace by \answerYes{}, \answerNo{}, or \answerNA{}.
    \item[] Justification: We provide a link to a public GitHub repository containing the source code used to run our method and produce results provided in this paper.
    This GitHub repository contains organized code that will allow any reader to run \ourmethod{} on all tasks in this paper. Additionally, the README in the repository provides detailed instructions to make setting up the proper environment and running the code easy for users. 
    \item[] Guidelines:
    \begin{itemize}
        \item The answer NA means that paper does not include experiments requiring code.
        \item Please see the NeurIPS code and data submission guidelines (\url{https://nips.cc/public/guides/CodeSubmissionPolicy}) for more details.
        \item While we encourage the release of code and data, we understand that this might not be possible, so “No” is an acceptable answer. Papers cannot be rejected simply for not including code, unless this is central to the contribution (e.g., for a new open-source benchmark).
        \item The instructions should contain the exact command and environment needed to run to reproduce the results. See the NeurIPS code and data submission guidelines (\url{https://nips.cc/public/guides/CodeSubmissionPolicy}) for more details.
        \item The authors should provide instructions on data access and preparation, including how to access the raw data, preprocessed data, intermediate data, and generated data, etc.
        \item The authors should provide scripts to reproduce all experimental results for the new proposed method and baselines. If only a subset of experiments are reproducible, they should state which ones are omitted from the script and why.
        \item At submission time, to preserve anonymity, the authors should release anonymized versions (if applicable).
        \item Providing as much information as possible in supplemental material (appended to the paper) is recommended, but including URLs to data and code is permitted.
    \end{itemize}

\item {\bf Experimental setting/details}
    \item[] Question: Does the paper specify all the training and test details (e.g., data splits, hyperparameters, how they were chosen, type of optimizer, etc.) necessary to understand the results?
    \item[] Answer: \answerYes{} % Replace by \answerYes{}, \answerNo{}, or \answerNA{}.
    \item[] Justification: All chosen hyper-parameters and implementation details are stated in \autoref{sec:experiments} and \cref{sec:detials}. 
    \item[] Guidelines:
    \begin{itemize}
        \item The answer NA means that the paper does not include experiments.
        \item The experimental setting should be presented in the core of the paper to a level of detail that is necessary to appreciate the results and make sense of them.
        \item The full details can be provided either with the code, in appendix, or as supplemental material.
    \end{itemize}

\item {\bf Experiment statistical significance}
    \item[] Question: Does the paper report error bars suitably and correctly defined or other appropriate information about the statistical significance of the experiments?
    \item[] Answer: \answerYes{} % Replace by \answerYes{}, \answerNo{}, or \answerNA{}.
    \item[] Justification: On all plots, we plot the mean taken over $20$ random runs and include error bars to show the standard error over the runs. 
    \item[] Guidelines:
    \begin{itemize}
        \item The answer NA means that the paper does not include experiments.
        \item The authors should answer "Yes" if the results are accompanied by error bars, confidence intervals, or statistical significance tests, at least for the experiments that support the main claims of the paper.
        \item The factors of variability that the error bars are capturing should be clearly stated (for example, train/test split, initialization, random drawing of some parameter, or overall run with given experimental conditions).
        \item The method for calculating the error bars should be explained (closed form formula, call to a library function, bootstrap, etc.)
        \item The assumptions made should be given (e.g., Normally distributed errors).
        \item It should be clear whether the error bar is the standard deviation or the standard error of the mean.
        \item It is OK to report 1-sigma error bars, but one should state it. The authors should preferably report a 2-sigma error bar than state that they have a 96\% CI, if the hypothesis of Normality of errors is not verified.
        \item For asymmetric distributions, the authors should be careful not to show in tables or figures symmetric error bars that would yield results that are out of range (e.g. negative error rates).
        \item If error bars are reported in tables or plots, The authors should explain in the text how they were calculated and reference the corresponding figures or tables in the text.
    \end{itemize}

\item {\bf Experiments compute resources}
    \item[] Question: For each experiment, does the paper provide sufficient information on the computer resources (type of compute workers, memory, time of execution) needed to reproduce the experiments?
    \item[] Answer: \answerYes{} % Replace by \answerYes{}, \answerNo{}, or \answerNA{}.
    \item[] Justification: All compute details are provided in \cref{sec:compute}.
    \item[] Guidelines:
    \begin{itemize}
        \item The answer NA means that the paper does not include experiments.
        \item The paper should indicate the type of compute workers CPU or GPU, internal cluster, or cloud provider, including relevant memory and storage.
        \item The paper should provide the amount of compute required for each of the individual experimental runs as well as estimate the total compute. 
        \item The paper should disclose whether the full research project required more compute than the experiments reported in the paper (e.g., preliminary or failed experiments that didn't make it into the paper). 
    \end{itemize}
    
\item {\bf Code of ethics}
    \item[] Question: Does the research conducted in the paper conform, in every respect, with the NeurIPS Code of Ethics \url{https://neurips.cc/public/EthicsGuidelines}?
    \item[] Answer: \answerYes{} % Replace by \answerYes{}, \answerNo{}, or \answerNA{}.
    \item[] Justification: We have made sure to adhere to the NeurIPS Code of Ethics in all aspects of our research.
    \item[] Guidelines:
    \begin{itemize}
        \item The answer NA means that the authors have not reviewed the NeurIPS Code of Ethics.
        \item If the authors answer No, they should explain the special circumstances that require a deviation from the Code of Ethics.
        \item The authors should make sure to preserve anonymity (e.g., if there is a special consideration due to laws or regulations in their jurisdiction).
    \end{itemize}

\item {\bf Broader impacts}
    \item[] Question: Does the paper discuss both potential positive societal impacts and negative societal impacts of the work performed?
    \item[] Answer: \answerYes{} % Replace by \answerYes{}, \answerNo{}, or \answerNA{}.
    \item[] Justification: See discussion of broader societal impacts in \cref{sec:impacts}. 
    \item[] Guidelines:
    \begin{itemize}
        \item The answer NA means that there is no societal impact of the work performed.
        \item If the authors answer NA or No, they should explain why their work has no societal impact or why the paper does not address societal impact.
        \item Examples of negative societal impacts include potential malicious or unintended uses (e.g., disinformation, generating fake profiles, surveillance), fairness considerations (e.g., deployment of technologies that could make decisions that unfairly impact specific groups), privacy considerations, and security considerations.
        \item The conference expects that many papers will be foundational research and not tied to particular applications, let alone deployments. However, if there is a direct path to any negative applications, the authors should point it out. For example, it is legitimate to point out that an improvement in the quality of generative models could be used to generate deepfakes for disinformation. On the other hand, it is not needed to point out that a generic algorithm for optimizing neural networks could enable people to train models that generate Deepfakes faster.
        \item The authors should consider possible harms that could arise when the technology is being used as intended and functioning correctly, harms that could arise when the technology is being used as intended but gives incorrect results, and harms following from (intentional or unintentional) misuse of the technology.
        \item If there are negative societal impacts, the authors could also discuss possible mitigation strategies (e.g., gated release of models, providing defenses in addition to attacks, mechanisms for monitoring misuse, mechanisms to monitor how a system learns from feedback over time, improving the efficiency and accessibility of ML).
    \end{itemize}
    
\item {\bf Safeguards}
    \item[] Question: Does the paper describe safeguards that have been put in place for responsible release of data or models that have a high risk for misuse (e.g., pretrained language models, image generators, or scraped datasets)?
    \item[] Answer: \answerNA{} % Replace by \answerYes{}, \answerNo{}, or \answerNA{}.
    \item[] Justification: The paper does not release new data or models with potential societal concerns.
    \item[] Guidelines:
    \begin{itemize}
        \item The answer NA means that the paper poses no such risks.
        \item Released models that have a high risk for misuse or dual-use should be released with necessary safeguards to allow for controlled use of the model, for example by requiring that users adhere to usage guidelines or restrictions to access the model or implementing safety filters. 
        \item Datasets that have been scraped from the Internet could pose safety risks. The authors should describe how they avoided releasing unsafe images.
        \item We recognize that providing effective safeguards is challenging, and many papers do not require this, but we encourage authors to take this into account and make a best faith effort.
    \end{itemize}

\item {\bf Licenses for existing assets}
    \item[] Question: Are the creators or original owners of assets (e.g., code, data, models), used in the paper, properly credited and are the license and terms of use explicitly mentioned and properly respected?
    \item[] Answer: \answerYes{} % Replace by \answerYes{}, \answerNo{}, or \answerNA{}.
    \item[] Justification: Creators of all assets used to produce all results in this paper are cited in \cref{sec:experiments}. All assets used are open source software or models. 
    \item[] Guidelines:
    \begin{itemize}
        \item The answer NA means that the paper does not use existing assets.
        \item The authors should cite the original paper that produced the code package or dataset.
        \item The authors should state which version of the asset is used and, if possible, include a URL.
        \item The name of the license (e.g., CC-BY 4.0) should be included for each asset.
        \item For scraped data from a particular source (e.g., website), the copyright and terms of service of that source should be provided.
        \item If assets are released, the license, copyright information, and terms of use in the package should be provided. For popular datasets, \url{paperswithcode.com/datasets} has curated licenses for some datasets. Their licensing guide can help determine the license of a dataset.
        \item For existing datasets that are re-packaged, both the original license and the license of the derived asset (if it has changed) should be provided.
        \item If this information is not available online, the authors are encouraged to reach out to the asset's creators.
    \end{itemize}

\item {\bf New assets}
    \item[] Question: Are new assets introduced in the paper well documented and is the documentation provided alongside the assets?
    \item[] Answer: \answerNA{} % Replace by \answerYes{}, \answerNo{}, or \answerNA{}.
    \item[] Justification: This paper does not introduce any new assets.
    \item[] Guidelines:
    \begin{itemize}
        \item The answer NA means that the paper does not release new assets.
        \item Researchers should communicate the details of the dataset/code/model as part of their submissions via structured templates. This includes details about training, license, limitations, etc. 
        \item The paper should discuss whether and how consent was obtained from people whose asset is used.
        \item At submission time, remember to anonymize your assets (if applicable). You can either create an anonymized URL or include an anonymized zip file.
    \end{itemize}

\item {\bf Crowdsourcing and research with human subjects}
    \item[] Question: For crowdsourcing experiments and research with human subjects, does the paper include the full text of instructions given to participants and screenshots, if applicable, as well as details about compensation (if any)? 
    \item[] Answer: \answerNA{} % Replace by \answerYes{}, \answerNo{}, or \answerNA{}.
    \item[] Justification: The work does not involve any human participants.
    \item[] Guidelines:
    \begin{itemize}
        \item The answer NA means that the paper does not involve crowdsourcing nor research with human subjects.
        \item Including this information in the supplemental material is fine, but if the main contribution of the paper involves human subjects, then as much detail as possible should be included in the main paper. 
        \item According to the NeurIPS Code of Ethics, workers involved in data collection, curation, or other labor should be paid at least the minimum wage in the country of the data collector. 
    \end{itemize}

\item {\bf Institutional review board (IRB) approvals or equivalent for research with human subjects}
    \item[] Question: Does the paper describe potential risks incurred by study participants, whether such risks were disclosed to the subjects, and whether Institutional Review Board (IRB) approvals (or an equivalent approval/review based on the requirements of your country or institution) were obtained?
    \item[] Answer: \answerNA{} % Replace by \answerYes{}, \answerNo{}, or \answerNA{}.
    \item[] Justification: The work does not involve any living participants.
    \item[] Guidelines:
    \begin{itemize}
        \item The answer NA means that the paper does not involve crowdsourcing nor research with human subjects.
        \item Depending on the country in which research is conducted, IRB approval (or equivalent) may be required for any human subjects research. If you obtained IRB approval, you should clearly state this in the paper. 
        \item We recognize that the procedures for this may vary significantly between institutions and locations, and we expect authors to adhere to the NeurIPS Code of Ethics and the guidelines for their institution. 
        \item For initial submissions, do not include any information that would break anonymity (if applicable), such as the institution conducting the review.
    \end{itemize}

\item {\bf Declaration of LLM usage}
    \item[] Question: Does the paper describe the usage of LLMs if it is an important, original, or non-standard component of the core methods in this research? Note that if the LLM is used only for writing, editing, or formatting purposes and does not impact the core methodology, scientific rigorousness, or originality of the research, declaration is not required.
    \item[] Answer: \answerNA{} % Replace by \answerYes{}, \answerNo{}, or \answerNA{}.
    \item[] Justification: This work does not use LLMs in any way. 
    \item[] Guidelines:
    \begin{itemize}
        \item The answer NA means that the core method development in this research does not involve LLMs as any important, original, or non-standard components.
        \item Please refer to our LLM policy (\url{https://neurips.cc/Conferences/2025/LLM}) for what should or should not be described.
    \end{itemize}

\end{enumerate}

%%%%%%%%%%%%%%%%%%%%%%%%%%%%%%%%%%%%%%%%%%%%%%%%%%%%%%%%%%%%
\newpage
\appendix
\tableofcontents
\newpage
\counterwithin{figure}{section} % reset figure numbering for each appendix section
\renewcommand{\thefigure}{\thesection.\arabic{figure}}
\counterwithin{table}{section}
\renewcommand{\thetable}{\thesection.\arabic{table}}
% \section{Appendix}
\section{Obtaining \textit{In Vitro} Minimal Inhibitory Concentration (MIC) Data }
\label{sec:marcelo-lab-methods}
In this section, we provide the methods used to produce all \textit{in vitro} minimal inhibitory concentration (MIC) values reported in this paper. 

\vspace{-1ex}
\paragraph{Peptide synthesis.} 
All peptides were synthesized by solid-phase peptide synthesis using the Fmoc strategy and purchased from AAPPTec.

\vspace{-1ex}
\paragraph{Bacterial strains and growth conditions used in the experiments. }
The following Gram-negative bacteria were used in our study: 
Acinetobacter baumannii ATCC 19606, Escherichia coli ATCC 11775, E. coli AIC221 (E. coli MG1655 phn$E_2$::FRT), E. coli AIC222 (E. coli MG1655 pmrA53 phn$E_2$::FRT (colistin resistant)), Klebsiella pneumoniae ATCC 13883, Pseudomonas aeruginosa PAO1, and P. aeruginosa PA14.
The following Gram-positive bacteria were also used in our study: Staphylococcus aureus ATCC 12600, S. aureus ATCC BAA-1556 (methicillin-resistant strain), Enetrococcus faecalis ATCC 700802 (vancomycin-resistant strain) and E. faecium ATCC 700221 (vancomycin-resistant strain). Bacteria were grown from frozen stocks and plated on Luria-Bertani (LB) or Pseudomonas isolation agar plates (P. aeruginosa strains) and incubated overnight at $\SI{37}{\degreeCelsius}$. 
After the incubation period, a single colony was transferred to $\SI{6}{\milli\liter}$ of LB medium, and cultures were incubated overnight ($\SI{16}{\hour}$) at $\SI{37}{\degreeCelsius}$. The following day, an inoculum was prepared by diluting the overnight cultures $1:100$ in $\SI{6}{\milli\liter}$ of the respective media and incubating them at $\SI{37}{\degreeCelsius}$ until bacteria reached logarithmic phase ($OD_{600} = 0.3 - 0.5$).

\vspace{-1ex}
\paragraph{Antibacterial assays.} 
The \textit{in vitro} antimicrobial activity of the peptides was assessed by using the broth microdilution assay \citep{labmethods}. Minimal inhibitory concentration (MIC) values of the peptides were determined with an initial inoculum of $\SI{2e6}{\cells\per\milli\liter}$ in LB in microtiter $\num{96}$-well flat-bottom transparent plates. Aqueous solutions of the peptides were added to the plate at concentrations ranging from $0.0625$ to $64$ $\si{\micro\mole\per\liter}$. 
The lowest concentration of peptide that inhibited 100 percent of the visible growth of bacteria was established as the MIC value in an experiment of $\SI{20}{\hour}$ of exposure at $\SI{37}{\degreeCelsius}$. The optical density of the plates was measured at $\SI{600}{\nano\meter}$ using a spectrophotometer. All assays were done as three biological replicates.

\label{appendix}

\newpage
\section{Additional Results}
\label{sec:more-results}
In this section, we provide all additional empirical results not included in the main text. 

\subsection{Additional Covering Sets of Solutions Found by \ourmethod{}}
In this section, we provide additional examples of covering sets of solutions found by \ourmethod{} for various tasks from \cref{sec:experiments}.

\begin{figure}[!ht]
\vskip 0.2in
\begin{center}
\centerline{\includegraphics[width=\columnwidth/2]{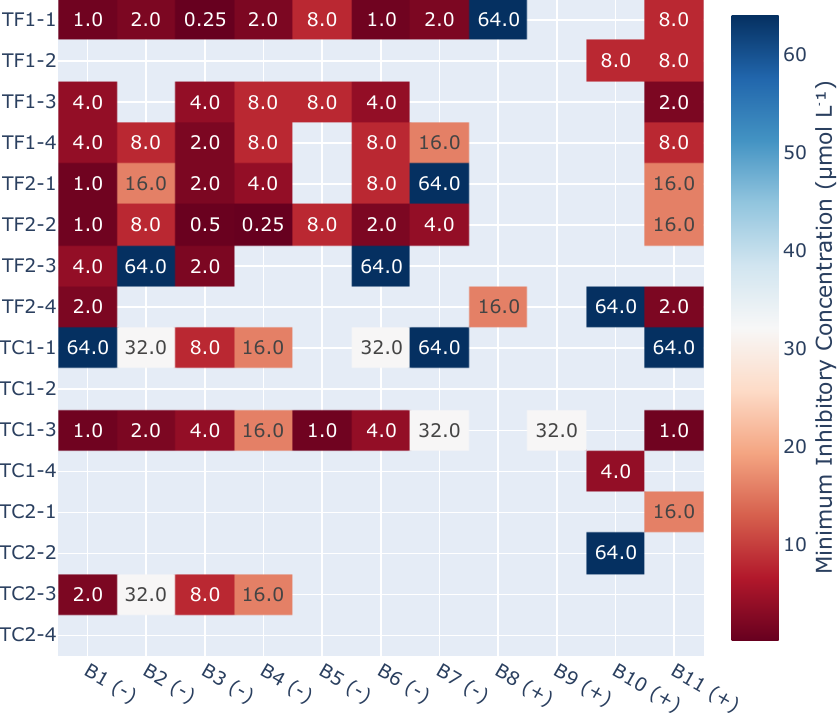}}
\caption{
\textit{In vitro} results for the two best ``template free" (TF1, TF2) and two best ``template constrained" (TC1, TC2) runs of \ourmethod{} for the peptide design task. 
Here, ``best" means runs that achieved highest coverage scores according to the APEX 1.1 model. 
Each row is a single peptide found by a run of \ourmethod{}. 
Row TF$i$-$j$ indicates template free run $i$, peptide $j$. Similarly, TC$i$-$j$ indicates template constrained run $i$, peptide $j$.
Columns are the \textit{in vitro} MICs for each target pathogenic bacteria B1$, \ldots,$ B11 listed in \cref{tab:bacteria}. (-) and (+) indicate Gram negative and Gram positive respectively. TF1 and TC1 correspond to the single runs of \ourmethod{} shown in \cref{tab:template-free-result} and \cref{tab:template-constrained-result} respectively. Methods used to obtain \textit{in vitro} MICs are provided in \cref{sec:marcelo-lab-methods}. 
}
\label{fig:lab-result-full}
\end{center}
\vskip -0.2in
\end{figure}

In \cref{fig:lab-result-full}, we provide \textit{In vitro} results for the two best ``template free" and two best ``template constrained" runs of \ourmethod{} for the peptide design task. 
Here, ``best" means runs that achieved highest coverage scores according to the APEX 1.1 model. 
\cref{fig:lab-result-full} provides \textit{in vitro} MICs for each of the $K=4$ peptides found by each of these runs of \ourmethod{}, for each of the $11$ target pathogenic bacteria. Methods used to obtain \textit{in vitro} MICs are provided in \cref{sec:marcelo-lab-methods}.

\begin{table}[!ht]
\caption{
The best set of $K=4$ peptide sequences found by one run of \ourmethod{} for the ``template free" peptide design task described in \cref{sec: tasks}. 
For each of the four sequences, we provide the MIC according to the APEX model for each of the $11$ target pathogenic bacteria.
The target pathogenic bacteria B1$, \ldots,$ B11 are listed in \cref{tab:bacteria}. (-) and (+) indicate Gram negative and Gram positive pathogenic bacteria respectively. 
The best/lowest MIC achieved for each pathogenic bacteria is in bold in each column.
See row ``TC1" in \cref{fig:lab-coverage-only}, and \cref{fig:lab-result-full} for \textit{in vitro} MICs for this set of $K=4$ peptides.}
\label{tab:template-constrained-result}
\centering
\resizebox{\columnwidth}{!}{
    \begin{tabular}{lccccccccccc}
        \toprule
        Peptide Amino Acid Sequence & B1(-) & B2(-) & B3(-) & B4(-) & B5(-) & B6(-) & B7(-) & B8(+) & B9(+) & B10(+) & B11(+) \\
        \midrule
        \texttt{KKLKIIRLLFK}  & 18.594 & 17.067 & 4.278 & 5.352 & \bf{13.460} & 50.442 & 24.543 & 456.831 & 431.276 & 441.292 & 20.305 \\
        \texttt{WAIRGLKLATWLSLNNKF} & 6.771 & 20.358 & 14.644 & 10.477 & 65.172 & 97.404 & 59.195 & \bf{19.846} & \bf{33.459} & 237.697 & 7.708 \\
        \texttt{RWARNLVRYVKWLKKLKKVI}  & \bf{2.171} & \bf{4.589} & \bf{2.641} & \bf{3.073} & 54.400 & \bf{11.444} & \bf{19.150} & 75.588 & 89.977 & 413.386 & \bf{2.913} \\
        \texttt{HWITIAFFRLSISLKI}  & 225.260 & 346.589 & 56.583 & 58.253 & 458.963 & 475.616 & 538.352 & 293.852 & 338.047 & \bf{34.230} & 22.153\\
        \bottomrule
    \end{tabular}
}
\end{table}

In \cref{tab:template-constrained-result}, we provide an example of a covering set of peptides found by \ourmethod{} for the ``template constrained" variation of the peptide design task.

In \cref{fig:church_example}, we provide the original HDR image (leftmost panel), and the $K=4$ images produced using the $K=4$ solutions found by a single run of \ourmethod{} for the church image variation of the image tone mapping task. 
An analogous result for the desk image variation of the image tone mapping task can be found in the main text in \cref{fig:desk-only}. 
A notable apparent limitation of the existing IAA/IQA metrics we used for the image tone mapping tasks is that they all favored monochromatic images for both the church and desk images. As such, the tone-mapped images obtained appear less colorful than their hand-tuned counterparts reported in prior work.

In \cref{tab:rano-result}, we provide an example of a covering set of $K=3$ molecules found by \ourmethod{} for the molecule design task. As mentioned in \cref{sec:experiments}, these three molecules effectively cover the $T=6$ objectives, as evidenced by the presence of all $T=6$ target elements in one of the $K=3$ molecules designed by \ourmethod{}. 

In \cref{fig:rover8} and \cref{fig:rover12}, we provide examples of a covering set of trajectories found for the $T=8$ and $T=12$ variations of the rover task respectively. In each figure, a panel is shown for each of the $T$ obstacles courses, with obstacles colored in red. 
The required starting point for the rover is a green point in the bottom left of each panel. The ``goal" end point that the rover aims to reach without hitting any obstacles is the red point in the top right of each panel. 
Each panel also shows the best among the $K$ trajectories found by a run of \ourmethod{} for navigating each obstacle course. 
An analogous plot for the $T=4$ variation of the rover optimization task is provided in the main text in \cref{fig:rover4-only}.

\begin{figure}[!ht]
\vskip 0.2in
\begin{center}
\centerline{\includegraphics[width=0.8\columnwidth]{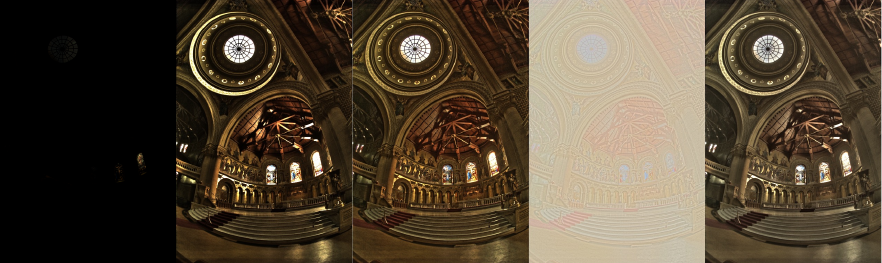}}
\caption{
\textbf{(Leftmost Panel)} Image obtained by naively compressing the dynamic range of the HDR church image in the ``image tone mapping" task. \textbf{(Four Rightmost Panels)} Images obtained by applying tone mapping to the church image using the best covering set of $K=4$ solutions found by a single run of \ourmethod{}. 
}
\label{fig:church_example}
\end{center}
\vskip -0.2in
\end{figure}
\begin{table*}[!ht]
\caption{The best set of $K=3$ molecules found by one run of \ourmethod{} for the Ranolazine MPO multiple element molecule design task described in \cref{sec: tasks}. For each of the three molecules, we provide the objective value obtained for each of the $T=6$ objectives. The best/highest objective value is in bold in each column. Each objective aims to add a different target element to Ranolazine (F, Cl, Br, Se, S, and P). The $T=6$ target elements are in bold in the SMILES string \citep{SMILES} representation of each molecule. }
\label{tab:rano-result}
\centering
\resizebox{\columnwidth}{!}{
    \begin{tabular}{lcccccc}
        \toprule
        Molecule (SMILES String) & Obj 1 (add F) & Obj 2 (add Cl) & Obj 2 (add Br) & Obj 2 (add Se) & Obj 2 (add S) & Obj 2 (add P) \\
        \midrule
        \midrule
\texttt{CC=C(C)C(OC(=O)C(O)CCCCCCC(=O)O)}\\\texttt{=CC=CCCCCCC[\textbf{Se}]CC(=O)NC1=CC=CC=C1C}  & 0.8038 & 0.8038 & 0.8038 & \bf{0.9108} & 0.8038 & 0.8038\\
\midrule
\texttt{CC=C(C)C(OC(=O)CCCCCCC(O)C(=\textbf{S})\textbf{Cl})}\\\texttt{=CC=CCOCCCCC(O)CC(=O)NC1=CC=CC=C1C}  & 0.8043 & \bf{0.9114} & 0.8043 & 0.8043 & \bf{0.9114} & 0.8043 \\
        \midrule
\texttt{CC=C(C)C(OC(=O)C(O)CCCCCCC(=O)C\textbf{Br})}\\\texttt{=CC=COC\textbf{P}CCCCN(C)CC(=O)[NH1]C1=CC=C(\textbf{F})C=C1C} & \bf{0.9097} & 0.8028 & \bf{0.9097} & 0.8028 & 0.8028 & \bf{0.9097} \\
        \bottomrule
    \end{tabular}
}
\end{table*}
\begin{figure}[!ht]
\vskip 0.2in
\begin{center}
\centerline{\includegraphics[width=0.8\columnwidth]{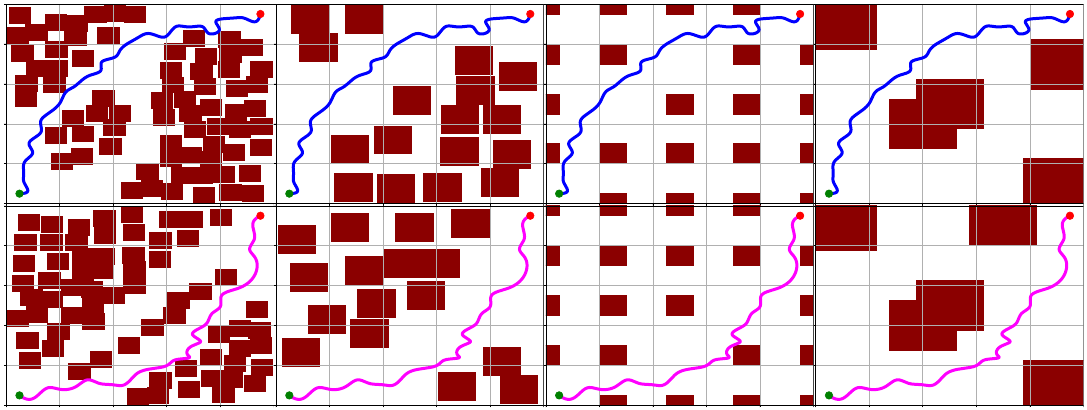}}
\caption{The 8 panels depict the 8 obstacle courses that the rover must navigate for the $T=8$ variation of the rover task, with obstacles colored in red.
The required starting point for the rover is a green point in the bottom left of each panel. The ``goal" end point that the rover aims to reach without hitting any obstacles is the red point in the top right of each panel. 
The line in each panel shows the best trajectory for navigating the obstacle course from among the $K=2$ covering trajectories found by a single run of \ourmethod{}.
The first trajectory in the covering set is shown in magenta and successfully navigates obstacle courses 5, 6, 7, and 8 \textbf{(Bottom Row)}. The second is shown in blue and successfully navigates obstacle courses 1, 2, 3, and 4 \textbf{(Top Row)}.} 
\label{fig:rover8}
\end{center}
\vskip -0.2in
\end{figure}
\begin{figure}[!ht]
\vskip 0.2in
\begin{center}
\centerline{\includegraphics[width=0.8\columnwidth]{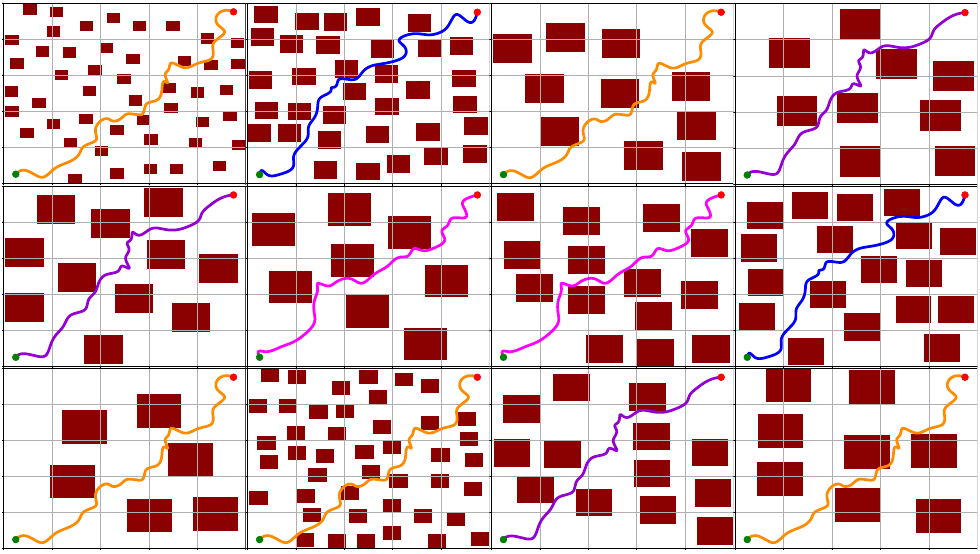}}
\caption{The 12 panels depict the 12 obstacle courses that the rover must navigate for the $T=12$ variation of the rover task, with obstacles colored in red. The required starting point for the rover is a green point in the bottom left of each panel. The ``goal" end point that the rover aims to reach without hitting any obstacles is the red point in the top right of each panel. 
The line in each panel shows the best trajectory for navigating each obstacle course from among the $K=4$ covering trajectories found by a run of \ourmethod{}.
The first trajectory in the covering set is shown in magenta and successfully navigates obstacle courses 6 and 7 \textbf{(Middle Row: Center Left and Center Right Panels)}. The second is shown in blue and successfully navigates obstacle courses 2 and 8 \textbf{(Top Row: Center Left Panel, Middle Row: Rightmost Panel)}. 
The third is shown in purple and successfully navigates obstacle courses 4, 5, and 11 \textbf{(Top Row: Rightmost Panel, Middle Row: Leftmost Panel, and Bottom Row: Center Right Panel)}.
The fourth trajectory is shown in orange and successfully navigates obstacle courses 1, 3, 9, 10, and 12 \textbf{(Top Row: Leftmost Panel and Center Right Panel, Bottom Row: Leftmost, Center Left, and Rightmost Panels)}.} 
\label{fig:rover12}
\end{center}
\vskip -0.2in
\end{figure}

\subsection{Additional Ablation Studies}
\label{sec:ablation2}
In this section, we provide additional ablation studies not included in the main text.  

\subsubsection{Ablation: Settings Where All Objectives are Highly Conflicting }
% \label{sec:ablation2}
\ourmethod{} is designed for cases where some, but not all, of the objectives are highly conflicting, and this is the case for all of the tasks considered in \cref{sec:experiments}. The conflicting pairs of the objectives prevent a single solution from optimizing all objectives well. The fact that not \textit{all} pairs of objectives are completely conflicting, is the reason why it is possible to cover all T objectives with a small set of K solutions. If all pairs of the T objectives are highly conflicting, there is by definition no possible set of $K<T$ solutions such that all objectives are well optimized. In this case, it’s best to use T individually optimized solutions if your goal is to find at least one solution that well optimizes each objective. However, it is still interesting to consider the performance of \ourmethod{} in the setting of T objectives that are pairwise highly conflicting. To investigate this, we construct a new variation of the multi-objective rover task described in \cref{sec: tasks}. We design T=3 obstacle courses such that it is impossible for the rover to take any single path that avoids all obstacles in any pair of the obstacle courses. This results in T=3 objectives that are pairwise highly conflicting. We then run 20 replications of \ourmethod{} on this problem with K=2, asking \ourmethod{} to design K=2 solutions that cover the T=3 pairwise highly conflicting objectives. For this task, \ourmethod{} got an average coverage score of $2.932$, while T individually optimized solutions got an average coverage score of $12.899$. It is unsurprising that T individually optimized solutions achieved better coverage here since all three objectives are pairwise highly conflicting. However, this result demonstrates that \ourmethod{} is able to design sets of K=2 solutions that still achieve fairly high coverage scores. 
\begin{figure}[!ht]
\vskip 0.2in
\begin{center}
\centerline{\includegraphics[width=0.8\columnwidth]{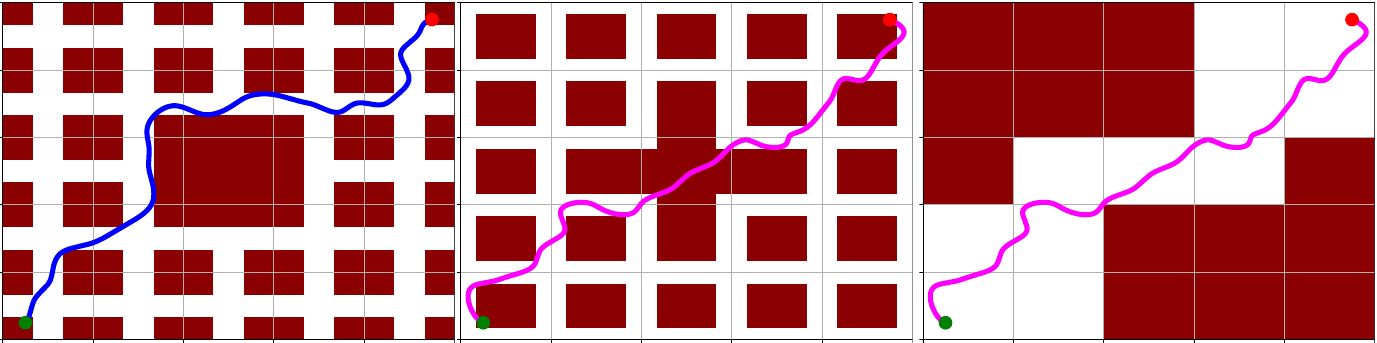}}
\caption{The three panels depict the T=3 obstacle courses that the rover must navigate for the ablation variation of the rover task where all T=3 objectives pairwise highly conflicting. Obstacles are colored in red. The three obstacle courses are designed such that any single path that successfully avoids all obstacles in one obstacle course, must hit some obstacles in both of the other obstacle courses. The required starting point for the rover is a green point in the bottom left of each panel. The ``goal" end point that the rover aims to reach without hitting any obstacles is the red point in the top right of each panel. The line in each panel shows the best trajectory for navigating the obstacle course from among the K=2 covering trajectories found by a single run of \ourmethod{}. The first trajectory in the covering set is shown in magenta and is the best trajectory for navigating obstacle courses 2 and 3 \textbf{(Middle Panel and Rightmost Panels)}. 
The second is shown in blue and successfully navigates obstacle course 1 \textbf{(Leftmost Panel)}.} 
\label{fig:rover3-ablation}
\end{center}
\vskip -0.2in
\end{figure}

\cref{fig:rover3-ablation} provides a diagram of the T=3 pairwise highly conflicting obstacle courses, and an example of one of the covering sets of K=2 trajectories found by a single run of \ourmethod{}. One optimized trajectory (shown in blue) is designed by \ourmethod{} such that it specializes to navigate the first obstacle course, navigating it without hitting any obstacles (Leftmost Panel). The other optimized trajectory (shown in magenta) is designed by \ourmethod{} such that it avoids all obstacles in the third obstacle course (Rightmost Panel), while also minimizing total amount of impact with the obstacles in the second obstacle course (Middle Panel). Since it is by-design impossible to avoid all obstacles in all T=3 obstacle courses with only K=2 solutions, \cref{fig:rover3-ablation} demonstrates that \ourmethod{} was able to successfully balance trade-offs among the objectives, designing a set of K=2 solutions that minimized the total amount of obstacle impact across the three obstacle courses.

\subsubsection{Ablation: \ourmethod{} Covering Set Size}
In this section, we ablate $K$, the user-specified hyperparameter that dictates of size of the set that \ourmethod{} designs to cover the $T$ objectives. For this ablation, we use the rover task with $T=4$ obstacle courses as defined in \cref{sec: tasks}. Note that in the main text we provide results comparing \ourmethod{} to baseline methods using $K=2$ for this task. In \cref{fig:k-ablation}, we provide results from running \ourmethod{} with each of $K=1, 2, 3,$ and $4$. With the ability to use more than $2$ solutions to cover the $T$ objectives ($K=3,4$), the optimization problem becomes easier and \ourmethod{} is able to converge more quickly. However, the loss in optimization efficiency inured by using the smaller covering set size of $K=2$, rather than a higher value of $K$, is marginal, highlighting \ourmethod{}'s ability of efficiency design smaller sets of high performing solutions. 

With only one solution ($K=1$), it is by definition not possible to cover all $T=4$ objectives since several pairs of the obstacle courses are specifically designed to be completely conflicting (no one trajectory can successfully avoid the obstacles in all $T=4$ obstacle courses). \ourmethod{} with $K=1$ therefore obtains a substantially lower final coverage score. 
\begin{figure}[!ht]
\vskip 0.2in
\begin{center}
\centerline{\includegraphics[width=0.8\columnwidth]{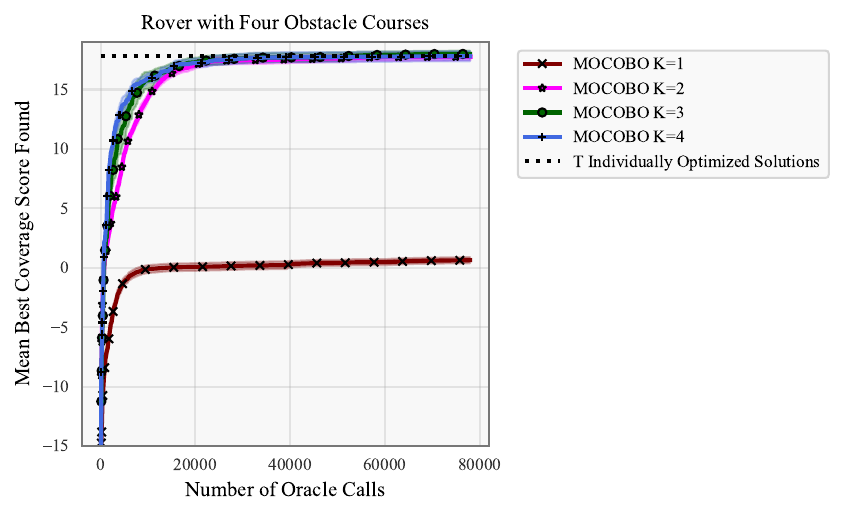}}
\caption{Ablating the hyperparameter $K$ used to run \ourmethod{} on the rover task with $T=4$ obstacle courses (see task definition in \cref{sec: tasks}).} 
\label{fig:k-ablation}
\end{center}
\vskip -0.2in
\end{figure}

\subsubsection{Ablation: \ourmethod{} Surrogate Model Quality}
In this section, we ablate the quality of the surrogate model used by \ourmethod{}. To run \ourmethod{} with surrogate models of varying quality, we vary the \ourmethod{} hyperparameter $m$: the number of inducing points used to define the approximate Gaussian process (GP) surrogate model (see surrogate model details in \cref{sec:detials}). It is well established in the literature that approximate GP models perform better with a larger number of inducing points, as the inducing point approximation used by the model is improved. However, there is an inherent trade-off as the computational cost of training the model increases with the number of inducing points. In practice, we therefore often to select the smallest possible value of $m$ (to maximize computational efficiency) such that we don't incur any significant performance degradation. 

For this ablation, we use the rover task with $T=4$ obstacle courses as defined in \cref{sec: tasks}. In \cref{fig:m-ablation}, we provide results from running \ourmethod{} with each of $m=4, 16, 64, 256, 1024,$ and $4096$. In all other experiments in this paper, we use $m=1024$. Results in \cref{fig:m-ablation} support our choice of $m=1024$ as results demonstrate that no significant performance improvement is gained by using the larger value of $m=4096$, and that performance starts to degrade with smaller values of $m \leq 256$. 
\begin{figure}[!ht]
\vskip 0.2in
\begin{center}
\centerline{\includegraphics[width=0.8\columnwidth]{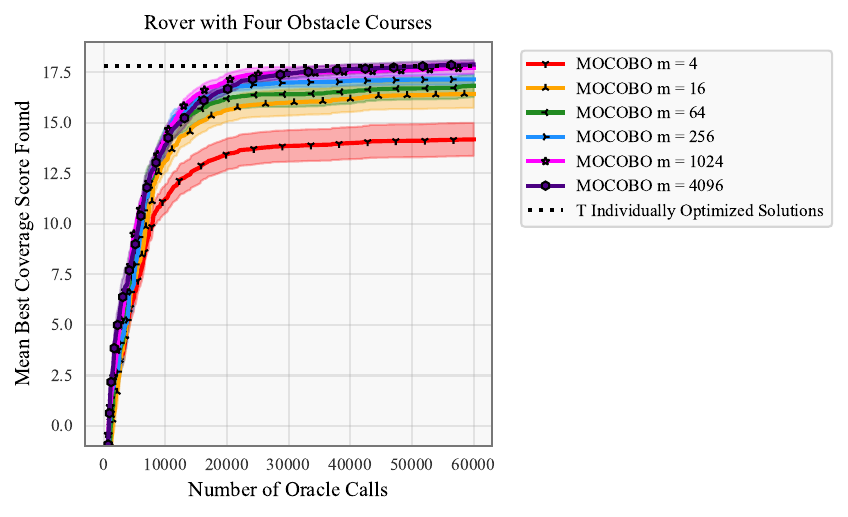}}
\caption{Ablating the \ourmethod{} hyperparameter $m$ (the number of inducing points used by the surrogate model) on the rover task with $T=4$ obstacle courses (see task definition in \cref{sec: tasks}). Lower values of $m$ correspond to lower surrogate model quality.} 
\label{fig:m-ablation}
\end{center}
\vskip -0.2in
\end{figure}

\subsubsection{Ablation: \ourmethod{} Batch Size}
In this section, we ablate the \ourmethod{} hyperparameter $q$, the acquisition batch size which dictates how many points are selected for evaluation from each trust region on each iteration of \ourmethod{}.
For this ablation, we use the rover task with $T=4$ obstacle courses as defined in \cref{sec: tasks}. In \cref{fig:q-ablation}, we provide results from running \ourmethod{} with each of $q=1, 2, 5, 10, 20,$ and $40$. In all other experiments in this paper, we use $q=20$. Results in \cref{fig:q-ablation} demonstrate the robustness \ourmethod{} to
changes in $q$, as there is little to no significant change in the performance of \ourmethod{} with different the values of $q$. 
\begin{figure}[!ht]
\vskip 0.2in
\begin{center}
\centerline{\includegraphics[width=0.8\columnwidth]{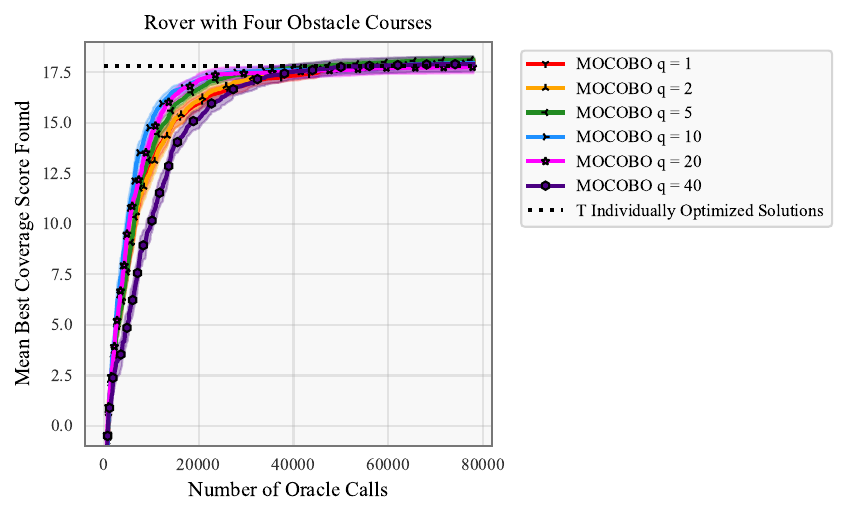}}
\caption{Ablating the \ourmethod{} hyperparameter $q$ (the number of points selected for evaluation during the acquisition step on each iteration) on the rover task with $T=4$ obstacle courses (see task definition in \cref{sec: tasks}).} 
\label{fig:q-ablation}
\end{center}
\vskip -0.2in
\end{figure}

\subsubsection{Ablation: \ourmethod{} Trust Regions}
In this section, we ablate the use of trust regions in \ourmethod{}. 
For this ablation, we use the rover task with $T=4$ obstacle courses as defined in \cref{sec: tasks}.
To ablate the use of trust regions, we compare to \ourmethod{} (as defined in \cref{sec: methods} with trust regions), to running \ourmethod{} without trust regions. 
Results in \cref{fig:eci-tr-ablation} demonstrate that \ourmethod{} performs significantly better with the use of trust regions. This result confirms that trust regions significantly improve performance in the high-dimensional settings we consider. 

\paragraph{Why trust regions} It has been well-established in the Bayesian optimization (BO) literature that standard BO (without trust regions or other high-dimensional adaptations) performs poorly in high dimensions. 
\citet{turbo} introduced trust regions as a principled way to improve performance of high-dimensional BO. Since then, trust regions have become a standard tool for any high-dimensional BO task. Since all of the tasks we consider in this paper are high-dimensional, this precisely why we adopt trust regions. We note that recent work has proposed alternative approaches to improve high-dimensional BO without trust regions (e.g., \cite{carlbo}). However, there remain concerns about the applicability of these approaches to structured domains (e.g., \cite{structured-bo-survey}). 
\begin{figure}[!ht]
\vskip 0.2in
\begin{center}
\centerline{\includegraphics[width=0.8\columnwidth]{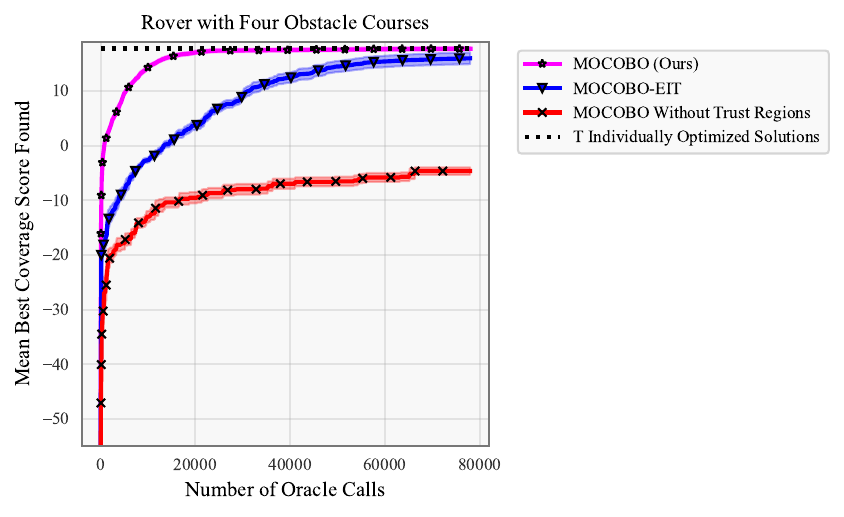}}
\caption{Ablation comparing \ourmethod{} as defined in \cref{sec: methods}, to two variations each removing a different component of \ourmethod{}. To ablate the use of trust regions in \ourmethod{}, we compare to \ourmethod{} without trust regions. To ablate our proposed ECI acquisition function, we additionally compare to \ourmethod{}-EIT, a variation of \ourmethod{} where we select candidates in acquisition by taking the maximum of standard EI for each of the $T$ objectives individually (EIT), rather than using ECI as described in \cref{sec: methods}. 
We provide optimization results for \ourmethod{} and these two variations on the rover task with $T=4$ obstacle courses (see task definition in \cref{sec: tasks}).
} 
\label{fig:eci-tr-ablation}
\end{center}
\vskip -0.2in
\end{figure}

\subsubsection{Ablation: \ourmethod{} Expected Coverage Improvement (ECI) Acquisition Function} 
In this section, we ablate our proposed expected coverage improvement (ECI) acquisition function.
For this ablation, we use the rover task with $T=4$ obstacle courses as defined in \cref{sec: tasks}.
To ablate the use of ECI, we compare to \ourmethod{} (as defined in \cref{sec: methods} with ECI), to \ourmethod{}-EIT: a variation of \ourmethod{} where on each iteration, we instead select points for evaluation by taking the maximum of the standard expected improvement (EI) acquisition function for each of the $T$ objectives individually. Results in \cref{fig:eci-tr-ablation} demonstrate that \ourmethod{} performs significantly better with our proposed ECI acquisition function, demonstrating the importance of ECI to the performance of \ourmethod{}. Additionally, \ourmethod{}-EIT still performed fairly well, outperforming all other baseline methods considered in the paper (e.g., all baselines in \cref{fig:main}). This highlights the importance of other aspects MOCOBO.

\subsection{Execution Time of Greedy Approximation Algorithm }
\label{sec:wallclock}
In \cref{tab:algo1-wall-time}, we provide an empirical evaluation of the average wall-clock runtime of \cref{alg:greedy-simple} with different values of $N$ (the number of data points), $K$ (the covering set size), and $T$ (the number of objectives). We provide average run times with $N = 20, 200, 2000, 20000, 200000,$ and $2000000$ for each combination of $K$ and $T$ that we ran \ourmethod{} on in \cref{sec:experiments}. To improve the speed of \cref{alg:greedy-simple}, in all experiments, we compute the inner for loop (lines 5-6 in \cref{alg:greedy-simple}) in parallel. This involves simply computing the marginal coverage of adding all $N$ data points at once in parallel, and then updating $A$ with the point that achieved maximum marginal converge improvement. Parallel computation of the marginal coverage improvement of the $N$ points is straightforward to implement using PyTorch. In \cref{tab:algo1-wall-time}, the ``Runtime" column gives the average runtime of \cref{alg:greedy-simple} with this parallel computation of the of the inner for loop, and the ``Runtime No Parallelization" column gives the average runtime without the parallel computation. Notice that as $N$ grows large, this parallel computation becomes essential to achieve the reasonably fast run times of \cref{alg:greedy-simple} needed to run \ourmethod{}. Average run times reported in \cref{tab:algo1-wall-time} are computed by running \cref{alg:greedy-simple} on a NVIDIA RTX A5000 GPU and averaging over $10$ runs. 
\begin{table}[!ht]
\caption{
Average wall-clock runtime of \cref{alg:greedy-simple} for different values of $N$ (the number of data points), $K$ (the covering set size), and $T$ (the number of objectives). 
The ``Runtime" column gives the average runtime of \cref{alg:greedy-simple} with parallel computation of the of the marginal coverage of adding the $N$ data points (parallel computation of lines 5-6). The ``Runtime No Parallelization" column gives the average runtime of \cref{alg:greedy-simple} without this parallelization of the inner for loop. 
Average run times are computed by running \cref{alg:greedy-simple} on a NVIDIA RTX A5000 GPU and averaging over $10$ runs. Standard errors over the $10$ runs are also provided. 
}
\label{tab:algo1-wall-time}
\centering
\resizebox{\columnwidth/2}{!}{
    \begin{tabular}{lcccc}
        \toprule
        N & K & T & Runtime No Parallelization (seconds) & Runtime (seconds) \\
        \midrule
        20  & 2 & 4 & $0.00605 \pm 0.00385$ & $0.00864 \pm 0.00795$ \\ 
        200  & 2 & 4 & $0.0230 \pm 0.00391$ & $0.00901 \pm 0.00828$ \\ 
        2000  & 2 & 4 & $0.194 \pm 0.00389$ & $0.00850 \pm 0.00781$ \\
        20000  & 2 & 4 & $2.0523 \pm 0.0259$ & $0.00852 \pm 0.00780$ \\
        200000  & 2 & 4 & $20.845 \pm 0.0446$ & $0.00872 \pm 0.00796$ \\
        2000000  & 2 & 4 & $195.740 \pm 0.114$ & $0.00923 \pm 0.00792$ \\
        \midrule
        20  & 2 & 8 & $0.00596 \pm 0.00375$ & $0.00906 \pm 0.00833$ \\ 
        200  & 2 & 8 & $0.0269 \pm 0.00446$ & $0.00894 \pm 0.00808$ \\
        2000  & 2 & 8 & $0.197 \pm 0.00386$ & $0.00865 \pm 0.00794$ \\
        20000  & 2 & 8 & $1.947 \pm 0.00348$ & $0.00864 \pm 0.00789$ \\
        200000  & 2 & 8 & $19.508 \pm 0.00964$ & $0.00915 \pm 0.00825$ \\
        2000000  & 2 & 8 & $193.317 \pm 0.218$ & $0.00962 \pm 0.00785$ \\
        \midrule
        20  & 4 & 12 & $0.00861 \pm 0.00430$ & $0.00873 \pm 0.00785$ \\
        200  & 4 & 12 & $0.0440 \pm 0.00371$ & $0.00875 \pm 0.00789$ \\
        2000  & 4 & 12 & $0.386 \pm 0.00642$ & $0.00910 \pm 0.00821$ \\
        20000  & 4 & 12 & $4.303 \pm 0.120$ & $0.00869 \pm 0.00776$ \\
        200000  & 4 & 12 & $38.596 \pm 0.0895$ & $0.00889 \pm 0.00784$ \\
        2000000  & 4 & 12 & $395.325 \pm 0.169$ & $0.0111 \pm 0.00772$ \\
        \midrule
        20  & 3 & 6 & $0.00704 \pm 0.00384$ & $0.00887 \pm 0.00809$ \\
        200  & 3 & 6 & $0.0334 \pm 0.00372$ & $0.00840 \pm 0.00764$ \\
        2000  & 3 & 6 & $0.314 \pm 0.00654$ & $0.00897 \pm 0.00817$ \\
        20000  & 3 & 6 & $2.942 \pm 0.00478$ & $0.00872 \pm 0.00785$ \\
        200000  & 3 & 6 & $28.732 \pm 0.0156$ & $0.00862 \pm 0.00775$ \\
        2000000  & 3 & 6 & $312.825 \pm 0.203$ & $0.00985 \pm 0.00799$ \\
        \midrule
        20  & 4 & 7 & $0.00765 \pm 0.00370$ & $0.00880 \pm 0.00793$ \\
        200  & 4 & 7 & $0.0448 \pm 0.00365$ & $0.00847 \pm 0.00760$ \\
        2000  & 4 & 7 & $0.412 \pm 0.00252$ & $0.00887 \pm 0.00798$ \\
        20000  & 4 & 7 & $3.917 \pm 0.00461$ & $0.00882 \pm 0.00789$ \\
        200000  & 4 & 7 & $38.932 \pm 0.0402$ & $0.00889 \pm 0.00791$ \\
        2000000  & 4 & 7 & $389.272 \pm 0.588$ & $0.0105 \pm 0.00809$ \\
        \midrule
        20  & 4 & 11 & $0.00772 \pm 0.00377$ & $0.00955 \pm 0.00853$ \\
        200  & 4 & 11 & $0.0422 \pm 0.00383$ & $0.00859 \pm 0.00775$ \\
        2000  & 4 & 11 & $0.387 \pm 0.00382$ & $0.00931 \pm 0.00837$ \\
        20000  & 4 & 11 & $4.155 \pm 0.00758$ & $0.00887 \pm 0.00792$ \\
        200000  & 4 & 11 & $41.869 \pm 0.0203$ & $0.00880 \pm 0.00777$ \\
        2000000  & 4 & 11 & $381.313 \pm 0.199$ & $0.0114 \pm 0.00803$ \\
        \bottomrule
    \end{tabular}
}
\end{table}

\paragraph{Efficiency of \cref{alg:greedy-simple}}
\cref{alg:greedy-simple} could also be made more efficient by pruning the points in $D_s$ with zero marginal coverage improvement after each outer-loop, since these points will continue to have zero marginal coverage improvement on subsequent loops. 
In future work, we plan to explore this and any other tricks that might allow us to further improve the efficiency of \cref{alg:greedy-simple}.

\newpage
\section{Batch Acquisition with Expected Coverage Improvement (ECI)}
\label{sec:approx-q-eci}
In batch acquisition, we select a batch of $q > 1$ candidates for evaluation.
In \cref{eq:full-q-eci}, we define q-ECI, a natural extension of the ECI acquisition function defined in \cref{eq:eci} to the batch acquisition setting. 
q-ECI gives the expected improvement in the coverage score after simultaneously observing the batch of $q$ points $\mathbf{X} = \left\{\bx_1, \ldots, \bx_q \right\}$.
When using batch acquisition, we aim to select a batch of $q$ points that maximize q-ECI. 

We will first discuss how one would estimate q-ECI using a Monte Carlo (MC) approximation. 
To select a batch of candidates $\hat{\mathbf{X}}$, we sample $m$ batches of $q$ points $B = \left\{ \mathbf{B}_{1}, \mathbf{B}_{2}, ..., \mathbf{B}_{m} \right\}$.
Here $\mathbf{B}_{j}$ is a batch of $q$ sampled points $\mathbf{B}_{j} = \left\{\bb_{j1}, \ldots, \bb_{jq} \right\}$. 
For each batch $\mathbf{B}_{j}$, we sample a realization $\hat{\mathbf{Y}}_{j} = \left\{\hat{\by}_{j1}, \ldots, \hat{\by}_{jq} \right\}$ from the GP surrogate model posterior. 
We leverage these samples to compute an MC approximation to the q-ECI of each $\mathbf{B}_{j}$: 
\begin{align}
q\text{-CI}(\mathbf{B}_{j}) &= \max(0, \max_{r=1, \dots, q} c(S^{*}_{D_s \cup \left\{(\bb_{jr}, \hat{\by}_{jr})\right\}}) - c(S^{*}_{D_s})).
\label{eq:q_coverage_improvement}
\end{align}
Here, $c(S^{*}_{D_s \cup \left\{(\bb_{jr}, \hat{\by}_{jr})\right\}})$ is the approximation of the coverage score of the new best covering set if we choose to evaluate candidate $\bb_{jr}$, assuming the candidate point will have the sampled objective values $\hat{\by}_{jr}$.
We would like to select and evaluate the batch of candidates $\mathbf{B}_{j}$ with the largest $q\text{-CI}$. 

Evaluating $q\text{-CI}$ for a single candidate batch requires $q$ evaluations of $c(S^{*}_{D_s \cup \left\{(\bb_{jr}, \hat{\by}_{jr})\right\}})$. 
Each evaluation of $c(S^{*}_{D_s \cup \left\{(\bb_{jr}, \hat{\by}_{jr})\right\}})$ requires a call to \cref{alg:greedy-simple} to first construct $S^{*}_{D_s \cup \left\{(\bb_{jr}, \hat{\by}_{jr})\right\}}$. 
Thus, batch acquisition with a full MC approximation of q-ECI requires $O(q \times m)$ calls of \cref{alg:greedy-simple}.
Assuming a sufficiently large $m$ to achieve a reliable MC approximation, this can become expensive for large batch sizes $q$. 
We therefore propose a faster approximation of batch ECI for practical use with large $q$. 

Instead of sampling $m$ batches of candidates, we sample $m$ individual data points $P = \left\{ \bp_{1}, \bp_{2}, ..., \bp_{m} \right\}$.
For each sampled point $\bp_{j}$, we sample a realization $\hat{\by}_{j} = (\hat{f}_1(\bp_{j}), \ldots, \hat{f}_T(\bp_{j}))$.
As in \cref{sec:eci}, we use the sampled realizations $\hat{\by}_{j}$ to compute an approximate coverage improvement $CI(\bp_{j})$ as defined in \cref{eq:coverage_improvement} for each point $\bp_{j}$.
To obtain a batch of $q$ candidates, we then greedily select the $q$ points $\bp_{j} \in P$ with the $q$ largest expected coverage improvements.
Note that this strategy does not involve sequential optimization of the $q$ points, as the batch of $q$ points is selected simultaneously as the points with the top-$q$ expected coverage improvements.

\newpage
\section{Limitations and Future Works}
\label{sec:limits}
\paragraph{Choosing $K$.} A primary limitation of \ourmethod{} is that the choice of the hyperparameter $K$ (the covering set size) may not always be straightforward. For example, as we mention in \cref{sec:experiments}, it may require domain knowledge to choose $K < T$ that is large enough that achieving good coverage is possible despite multiple conflicting objectives. 
In many practical applications, we prefer the smallest possible set size $K$ such that good coverage can still be achieved.
In future work, we plan to explore methods for simultaneously optimizing both the solutions in the covering set, and the size of the covering set, balancing the trade-off between minimizing the number of solutions needed and maximizing overall coverage.

\paragraph{Unsupervised generative model pre-training for structured domains.} We note that applying \ourmethod{} to structured domains requires a pre-trained generative model, such as a Variational Autoencoder (VAE), to embed the discrete input space into a continuous latent space where \ourmethod{} can be directly applied. Training such generative models typically demands significant computational resources and a large corpus of unlabeled data, which may not always be available in new domains. For the two structured tasks considered in this paper\textemdash{}molecule and peptide design\textemdash{}we leverage publicly available pre-trained VAEs from prior work. In contrast, \ourmethod{} does not require a generative model for continuous input spaces, where optimization is performed directly in the original domain. This distinction highlights a key practical limitation: deploying \ourmethod{} in new structured domains would necessitate pre-training a new generative model, which may be a barrier in resource-constrained settings.

\paragraph{Exploring lazy greedy evaluation approaches to improve computational efficiency.} Another avenue for future work is improving the computational efficiency of \cref{alg:greedy-simple} by using a lazy greedy evaluation strategy. Commonly employed in submodular maximization, lazy greedy approaches maintain a priority queue of marginal gains and only recompute them when necessary, avoiding redundant evaluations and reducing total computational cost. 

\paragraph{Threshold-based coverage.} 
The coverage optimization problem we consider in this paper provides a principled way to pose the goal of finding a small set of $K$ solutions such that each of the $T$ objectives is optimized \textit{as much as possible} by at least one solution in the set. This captures the setting we care about, for example, discovering $K$ antibiotics such that each pathogen is targeted as effectively as possible, not just adequately. 
In contrast, threshold-based coverage constitutes a fundamentally different problem: it assumes that we know, a priori, a threshold value for each objective beyond which we do not care to improve performance further. 
This setting is not aligned with the domains we focus on in this paper where better objective values are always desirable and thresholds are typically unknown or unhelpful. 
However, threshold-based coverage is a meaningful formulation for many other domains, such as drug toxicity screening, where desired minimum performance thresholds are known ahead of time. 
To address such domains, we plan to explore this alternative problem setting of threshold-based coverage optimization in future work.

\section{Broader Impact}
\label{sec:impacts}
This research includes applications in molecule and peptide design. While AI-driven biological design holds great promise for benefiting society, it is crucial to acknowledge its dual-use potential. Specifically, AI techniques designed for drug discovery could be misused to create harmful biological agents \cite{dualuse}.

Our goal is to accelerate drug development by identifying promising candidates, but it is imperative that experts maintain oversight, that all potential therapeutics undergo thorough testing and clinical trials, and that strict regulatory frameworks governing drug development and approval are followed.

\newpage
\section{Compute Resources}
\label{sec:compute}
In this section, we provide all details about the compute resources used to produce all results in this paper. 
\begin{table}[H]
  \centering
  \begin{threeparttable}
  \caption{Setup of internal cluster used to run experiments.}\label{table:internalcluster}
  \begin{tabular}{ll}
    \toprule
    \multicolumn{1}{c}{\textbf{Type}}
    & \multicolumn{1}{c}{\textbf{Specifications}}
    \\ \midrule
    System Topology & 20 nodes with 2 sockets each with 24 logical threads (total 48 threads) \\
    Processor       & 1 Intel Xeon Silver 4310, 2.1 GHz (maximum 3.3 GHz) per socket \\
    Cache           & 1.1 MiB L1, 30 MiB L2, and 36 MiB L3 \\
    Memory          & 250 GiB RAM \\
    Accelerator     & 1 NVIDIA RTX A5000 per node, 2 GHZ, 24GB RAM 
    \\ \bottomrule
  \end{tabular}
  \end{threeparttable}
\end{table}
\paragraph{Compute specifications (type and memory).}
We use GPU works to run all experiments and produce all empirical results provided in this paper.
A single GPU was used per run of each method compared on each task. 
Each each run uses approximately 12-18 GB of the GPU memory. 
Most experiments were executed on our internal cluster of NVIDIA RTX A5000 GPUs (see internal cluster compute details in \Cref{table:internalcluster}). 
We also used cloud compute resources for two weeks to complete additional replications of some experiments.  
We used a total of eight RTX 4090 GPU workers from \texttt{runpod.io}, each with approximately 24 GB of GPU memory. 

\paragraph{Execution time.}
For the relatively inexpensive rover task, each optimization run takes approximately 1 day of execution time.
For all other tasks considered, each optimization run takes approximately 3 days of execution time.
To create all coverage optimization plots, we ran all methods compared $20$ times each. 
Completing all of the runs needed to produce all of the results in this paper required roughly $64000$ total GPU hours. 

\paragraph{Compute resources for preliminary experiments.}
Preliminary experiments refer to the initial experiments for e.g. method development that are not included as results in the paper. 
All preliminary experiments were run on our internal cluster of NVIDIA RTX A5000 GPUs (see internal cluster compute details in \Cref{table:internalcluster}). 
We spent approximately $2000$ hours of GPU time on preliminary experiments.

\newpage
\section{Additional Implementation Details}
\label{sec:detials}
In this section, we provide additional implementation details for \ourmethod{}. We also refer readers to the \ourmethod{} codebase for the full-extent of implementation details and experimental setup needed to reproduce results provided \url{https://github.com/nataliemaus/mocobo}.

\subsection{Trust Region Hyperparameters}
\label{sec:tr-hypers}
For all trust region methods, the trust region hyperparameters are set to the \turbo{} defaults used by \citet{turbo}. 

\subsection{Surrogate Model}
\label{sec:ppgpr}
Since the tasks considered in this paper are challenging, high-dimensional tasks requiring a large number of function evaluations, we use approximate Gaussian process (GP) surrogate models. 
In particular, we use Parametric Gaussian Process Regressor (PPGPR)~\cite{PPGPR} surrogate models with a constant mean, standard RBF kernel, and $1024$ inducing points. Additionally, we use a deep kernel (several fully connected layers between the search space and the GP kernel) \citep{dkl}. We use two fully connected layers with $D$ nodes each, where $D$ is the dimensionality of the search space. 

We use the same PPGPR model(s) with the same configuration for \ourmethod{}, \turbo{}, \lolbo{}, and \robot{}. 
For \ourmethod{}, to model the $T$-dimensional output space, we use $T$ PPGPR models, one to approximate each objective $f_1, \ldots, f_T$. To allow information sharing between the models, we use a shared deep kernel (the $T$ PPGPR models share the same two-layer deep kernel) \citep{shared-dkl, dkt}. 

Unlike the other methods compared, \morbo{} was designed for use with an exact GP model rather than an approximate GP surrogate model. For fair comparison, we therefore run \morbo{} with an exact GP using all default hyperparameters and the official codebase provided by \citet{morbo}.

We train the PPGPR surrogate model(s) on data collected during optimization using the Adam optimizer~\cite{adam-optimizer} with a learning rate of $0.001$ and a mini-batch size of $256$. On each step of optimization, we update the model on collected data until we stop making progress (loss stops decreasing for $3$ consecutive epochs), or exceed $30$ epochs. 
Since we collect a large amount of data for each optimization run (e.g., as many as $2e6$ data points in a single run for the ``template constrained" peptide design task), we avoid updating the model on all data collected at once. On each step of optimization, we update the current surrogate model only on a subset of $1000$ of the collected data points.
This subset is constructed from the data that has obtained the highest objective values so far, along with the most recent batch of data collected. 
By always updating on the most recent batch of data collected, we ensure that the surrogate model is conditioned on every data point collected at some point during the optimization run.

\subsection{Initialization Data} 
\label{sec: init-data}
In this section, we provide details regarding the data used to initialize all optimization runs for all tasks in \cref{sec:experiments}. 

To initialize optimization for the molecule design task, we take a random subset of $10000$ molecules from the standardized unlabeled dataset of 1.27M molecules from the Guacamol benchmark software \cite{GuacaMol}.
We generate labels for these $10000$ molecules once, and then use the labeled data to initialize optimization for all methods compared. 

To initialize optimization for the peptide design tasks, we generate a a set of $20000$ peptide sequences by making random edits (insertions, deletions, and mutations) to the $10$ template peptide sequences in \cref{tab:templates}. We generate labels for these $20000$ peptides once, and then use the labeled data to initialize optimization for all methods compared. 

For all other tasks, we initialize optimization with $2000$ points sampled uniformly at random from the search space. 

\subsection{Diversity Constraints and Associated Hyperparameters for the \robot{} Baseline} 
\label{sec: robot-diversity-hypers}
For a single objective, \robot{} seeks a diverse set of $M$ solutions, requiring that the set of solutions have a minimum pairwise diversity $\tau$ according to the user specified diversity function $\divf{}$. 
Since \citet{robot} also consider rover and molecule design tasks, we use the same diversity function $\divf{}$ and diversity threshold $\tau$ used by \citet{robot} for these two tasks. For the peptide design tasks, we define $\divf{}$ to be the edit distance between peptide sequences, and use a diversity threshold of $\tau=3$ edits. For the image optimization task, since there is no obvious semantically meaningful measure of diversity between two sets of input parameters, we define $\divf{}$ to be the Euclidean distance between solutions, and use $\tau = 1.45$, the approximate average Euclidean distance between a randomly selected pair of points in the search space.

\newpage
\section{Additional Task Details}
\label{sec:task-detials}
In this section we provide additional details for the chosen set of tasks we provide results for in \cref{sec:experiments}. 

\subsection{HDR Image Tone Mapping}
\label{sec:image-task-more-detials}
\begin{table}[!ht]
\centering
\caption{\texttt{pyiqa} metric ID strings used to identify the $T=7$ target image quality metrics used for image tone mapping. Each metric is a no-reference image aesthetic (IAA) or quality (IQA) assessment metric obtained from the \texttt{pyiqa} library~\citep{chaofeng2022iqapytorch}.}
\label{tab:img-metrics}
    \begin{tabular}{rll}
        \toprule
        \multicolumn{1}{c}{Objective ID} & \multicolumn{1}{c}{Pyiqa Metric ID} & \multicolumn{1}{c}{Reference} \\
        \midrule
        1  & \texttt{nima} & \citealp{talebi2018nima} \\
        2 & \texttt{nima-vgg16-ava} & \citealp{talebi2018nima,murray2012ava} \\
        3 & \texttt{topiq-iaa-res50} & \citealp{chen2024topiq} \\
        4 & \texttt{laion-aes} & \citealp{schuhmann2022laion} \\
        5 & \texttt{hyperiqa} & \citealp{su2020blindly} \\
        6 & \texttt{tres} & \citealp{golestaneh2022noreference} \\
        7 & \texttt{liqe} & \citealp{zhang2023cvpr} \\
        \bottomrule
    \end{tabular}
\end{table}
In \cref{tab:img-metrics}, we list the names of the $7$ image quality metrics used for the image tone mapping tasks described in \cref{sec: tasks}.

\vspace{-1ex}
\paragraph{Target metrics.}
The target image quality (Image Quality Assessment, IQA) and aesthetic (Image Aesthetic Assessment, IAA) metrics are organized in \cref{tab:img-metrics}, where all except \texttt{nima} are IAA metrics, while \texttt{nima} is an IQA metric.
For more detailed information about each metric and their corresponding datasets, please refer to their original references.

\vspace{-1ex}
\paragraph{Benchmark images.}
We use two benchmark images. The first is the ``Stanford Memorial Church" image obtained from \url{https://www.pauldebevec.com/Research/HDR/} by courtesy of Paul E. Debevec~\citep{1997recovering}. 
The second is the ``desk lamp" image obtained from \url{https://cadik.posvete.cz/tmo/} by courtesy of Martin \v{C}ad\'{i}k~\citep{cadik2008phd}.
Because commonly used metrics are only correlated with subjective image quality, prior work on tuning parameters for these and related benchmarks has been done by trial and error and human-in-the-loop type schemes~\citep{lischinski2006interactive}, including preferential BO-based approaches~\citep{koyama2017sequential,koyama2020sequential}.

\vspace{-1ex}
\paragraph{Imaging pipeline.}
We consider a tone mapping pipeline consisting of a multi-layer detail decomposition~\citep{tumblin1999lcis,durand2002fast,li2005compressing} using the guided filter by~\citet{he2013guided} (3 detail layers and 1 base layer), followed by gamma correction~\citep[Section 2.9]{reinhard2005high}, resulting in a 13-dimensional optimization problem. 
The complete image processing pipeline is very similar to the classic approach proposed by \citet{tumblin1999lcis}, where the main difference is that, similarly to \citet{farbman2008edge}, we replace the diffusion smoothing filter with a more recent edge-preserving detail smoothing filter, the guided filter, by~\citet{he2013guided}.
First, given an HDR image in the RGB color space \(I = (I_{\mathrm{r}}, I_{\mathrm{r}}, I_{\mathrm{b}}))\), where \(I_{\mathrm{r}}\) is the red channel, \(I_{\mathrm{r}}\) is the blue channel, and \(I_{\mathrm{g}}\) is the green channel, we compute the luminance according to 
\[
    L \triangleq 0.2989 I_r +  0.587 I_g + 0.114 I_b \, .
\]
(The constants were taken from the code of \citealt{farbman2008edge}.)
The luminance image \(L\) is then logarithmically compressed and then decomposed into three detail layers and one base layer by applying the guided filter three times, each with a different radius parameter \(r_i\) and smoothing parameter \(\epsilon_i\) for \(i = 1, \ldots, 3\).
Then, we amplify or attenuate each channel with a corresponding gain coefficient \(g_{\mathrm{detail}, i}\) \(i = 1, \ldots, 3\) for the detail and \(g_{\mathrm{base}}\) for the base layers.
The image is then reconstructed by adding all the layers, including the colors.
Following~\citet{tumblin1999lcis}, we also apply a gain, \(g_{\mathrm{color}}\), to the color channels.
Finally, the resulting image is applied an overall gain \(g_{\mathrm{out}}\) and then gamma-corrected~\citep[Section 2.9]{reinhard2005high} with an exponent of \(1/\gamma\).
The parameters for this pipeline are organized in \cref{tab:img-parameters}.
The implementation uses OpenCV~\citep{Bradski2000opencv}, in particular, the guided filter implementation in the extended image processing (\texttt{ximgproc}) submodule.
\begin{table}[!ht]
\centering
\caption{\texttt{pyiqa} metric ID strings used to identify the $T=7$ target image quality metrics used for the image tone mapping task.}
\label{tab:img-parameters}
\begin{tabular}{rll}
    \toprule
    \multicolumn{1}{c}{Parameter} & \multicolumn{1}{c}{Description} & \multicolumn{1}{c}{Domain} \\
    \midrule
    \(r_i\) & radius of the guided filter for generating the \(i\)th detailed layer & \(\{3, \ldots, 32\}\)  \\
    \(\epsilon_i\)  & \(\epsilon\) of the guided filter for generating the \(i\)th detailed layer & \([0.01, 10]\) \\
    \(g_{\mathrm{detail}, i}\)  &  Gain of the \(i\)th detail layer & \([0, 1.5]\) \\
    \(g_{\mathrm{base}}\)  &  Gain of the \(i\)th detail layer & \([0, 1]\) \\
    \(g_{\mathrm{color}}\)  &  Gain of the color layer &  \([0.5, 1.5]\) \\
    \(g_{\mathrm{out}}\)  &  Gain of tone-mapped output & \([0.2, 2.0]\)  \\
    \(\gamma\)  &  Gamma correction inverse exponent & \([1, 5]\)  \\
    \bottomrule
\end{tabular}
\end{table}
\vspace{-1ex}
\paragraph{Optimization problem setup.}
For optimization, we map the parameters 
\[ 
    \mathbf{x} = \left( r_1, \, r_2, \, r_3, \, \epsilon_1, \, \epsilon_2, \, \epsilon_3, \, g_{\mathrm{detail},1}, \, g_{\mathrm{detail},2}, \, g_{\mathrm{detail},3}, \, g_{\mathrm{base}}, \, g_{\mathrm{color}}, \, g_{\mathrm{out}}, \, \gamma \right) 
\]
to the unit hypercube \({[0, 1]}^{13}\).
In particular, the mapped values on the unit interval \([0, 1]\) linearly interpolate the domain of each parameter shown in \cref{tab:img-parameters}.
For the radius parameters \(r_1, r_2, r_3\), which are categorical, we naively quantize the domain by rounding the output of the interpolation to the nearest integer.

\subsection{Peptide Design}
\label{sec:peptides-task-more-detials}
\begin{table}[!ht]
\centering
\caption{Names of the $T=11$ target bacteria used for the peptide design task. The first seven bacteria are Gram negative (IDs B1-B7) and the last four (IDs B8-B11) are Gram positive.}
    \label{tab:bacteria}
\resizebox{\columnwidth/2}{!}{
    \begin{tabular}{lc}
        \toprule
        Objective ID & Target Pathogenic Bacteria \\
        \midrule
        B1  & \texttt{A. baumannii ATCC 19606} \\
        B2 & \texttt{E. coli ATCC 11775} \\
        B3 & \texttt{E. coli AIC221} \\
        B4 & \texttt{E. coli AIC222-CRE} \\
        B5 & \texttt{K. pneumoniae ATCC 13883} \\
        B6 & \texttt{P. aeruginosa PAO1} \\
        B7 & \texttt{P. aeruginosa PA14} \\
        B8 & \texttt{S. aureus ATCC 12600} \\
        B9 & \texttt{S. aureus ATCC BAA-1556-MRSA} \\
        B10 & \texttt{E. faecalis ATCC 700802-VRE} \\
        B11 & \texttt{E. faecium ATCC 700221-VRE} \\
        \bottomrule
    \end{tabular}
}
\end{table}
\begin{table}[!ht]
\centering
\caption{Template amino acid sequences used for the ``template constrained" peptide design task.}
\label{tab:templates}
\resizebox{\columnwidth/3}{!}{
    \begin{tabular}{c} 
        \toprule
        Template Amino Acid Sequences \\
        \midrule
        \texttt{RACLHARSIARLHKRWRPVHQGLGLK} \\
        \texttt{KTLKIIRLLF} \\
        \texttt{KRKRGLKLATALSLNNKF} \\
        \texttt{KIYKKLSTPPFTLNIRTLPKVKFPK} \\
        \texttt{RMARNLVRYVQGLKKKKVI} \\
        \texttt{RNLVRYVQGLKKKKVIVIPVGIGPHANIK} \\
        \texttt{CVLLFSQLPAVKARGTKHRIKWNRK} \\
        \texttt{GHLLIHLIGKATLAL} \\
        \texttt{RQKNHGIHFRVLAKALR} \\
        \texttt{HWITINTIKLSISLKI} \\
        \bottomrule
    \end{tabular}
}
\end{table}
\cref{tab:bacteria} specifies the $T=11$ target bacteria used for the peptide design task from \cref{sec:experiments}. The first seven bacteria are Gram negative bacteria (Objective IDs B1-B7) and the last four (Objective IDs B8-B11) are Gram positive.
\cref{tab:templates} gives the $10$ template amino acid sequences used for the ``template constrained" variation of the peptide design task from \cref{sec:experiments}.

\newpage
\section{Proof that Finding the Best Observed Covering Set is NP-hard}
\label{sec: nphard}
In this section, we prove \cref{lemma:np_hardness}, that finding $S^{*}_{D_s}$ is NP-Hard. 

\begin{proof}
We prove \cref{lemma:np_hardness} by reduction from the well-known \emph{Maximum Coverage Problem (MCP)}, which is NP-Hard.

\begin{definition}[Maximum Coverage Problem (MCP)]
\label{def:mcp}
In the Maximum Coverage Problem, we are given:
\begin{itemize}
    \item A universe \( U = \{e_1, e_2, \ldots, e_m\} \) of \( m \) elements.
    \item A collection of \( n \) subsets \( \mathcal{S} = \{A_1, A_2, \ldots, A_n\} \), where \( A_i \subseteq U \).
    \item An integer \( K \), the number of subsets we can select.
\end{itemize}
The objective is to find a collection of \( K \) subsets \( \mathcal{S}' \subseteq \mathcal{S} \) such that the total number of elements covered, \( \bigcup_{A \in \mathcal{S}'} A \), is maximized.
\end{definition}

\begin{proposition}[MCP is NP-Hard]
\label{prop:mcp_hard}
The Maximum Coverage Problem is NP-Hard.
\end{proposition}

\paragraph{Reduction from MCP to Finding \( S^{*}_{D_s} \):}
We reduce an instance of MCP to the problem of finding \( S^{*}_{D_s} \) as follows:
\begin{enumerate}
    \item Let the universe \( U = \{e_1, e_2, \ldots, e_m\} \) correspond to the objectives \( \{1, 2, \ldots, T\} \) in the optimal covering set problem, i.e., set \( T = m \).
    \item Let each subset \( A_i \in \mathcal{S} \) correspond to a point \( \bx_i \in D_s \).
    \item Define each objective \( f_t : \inputdom \rightarrow \{0,1\} \), and set \( f_t(\bx_i) = 1 \) if subset \( A_i \) contains element \( e_t \), and \( f_t(\bx_i) = 0 \) otherwise. Intuitively, this means that each point \( \bx_i \) "covers" objective \( f_t \) if it achieves value 1 under that objective. While the design points $x_i$ are not literal subsets, they induce coverage behavior that mirrors the structure of MCP through their binary function values across the objectives.
    \item Under this mapping, coverage score $c(S) = \sum_{t=1}^T \max_{\bx \in S} f_t(\bx)$ is the total number of objectives “covered” (i.e., the number of objectives for which at least one of the selected points has value 1). It follows that the goal of the Maximum Coverage Problem (selecting $K$ ``subsets" to maximize coverage) corresponds exactly to selecting $K$ points \( \bx_i \in D_s \) to maximize the coverage score $c(S)$. 
\end{enumerate}

\paragraph{Correctness of the Reduction:}
The reduction ensures that:
\begin{itemize}
    \item Each subset \( A_i \in \mathcal{S} \) is encoded by a design point \( \bx_i \in D_s \) via its binary-valued outputs over the objectives.

    \item Each element \( e_t \in U \) is mapped to objective \( f_t \), and is considered “covered” if some selected point \( \bx \in S \) satisfies \( f_t(\bx) = 1 \). In particular, if subset \( A_i \) covers/contains element \( e_t \), then \( f_t(\bx_i) = 1 \); otherwise, \( f_t(\bx_i) = 0 \).
    
    \item The MCP objective (maximize number of covered elements) is equivalent to maximizing the coverage score \( c(S) \), which counts how many objectives are covered by the selected set \( S \).
\end{itemize}
Thus, solving the optimal covering set problem is equivalent to solving MCP.

\paragraph{Implications:}
Since MCP is NP-Hard (Proposition~\ref{prop:mcp_hard}), and we have reduced MCP to the problem of finding \( S^{*}_{D_s} \) in polynomial time, it follows that finding \( S^{*}_{D_s} \) is also NP-Hard.

\end{proof}

\paragraph{Proof significance}
We do not claim the above proof of \cref{lemma:np_hardness} as a novel contribution of this work, as it follows straightforwardly from the fact that the well-known Maximum Coverage Problem is NP-hard. 
Note that proving \cref{lemma:np_hardness} would also follow straightforwardly from the work of \citet{icml25-r1-1} who demonstrate that the coverage optimization problem generalizes k-means clustering. 
Rather than providing additional novel contribution, the above proof of \cref{lemma:np_hardness} serves to justify our use of a greedy approximation algorithm (\cref{alg:greedy-simple}) to approximate $S^{*}_{D_s}$ on each iteration of \ourmethod{}.

\newpage
\section{Approximation Proof for Greedy Algorithm} 
\label{sec: greedyproof}
In this section, we prove \cref{theorem:greedy-simple-approx}, that \cref{alg:greedy-simple} is a $(1 - \frac{1}{e})$-Approximation.

\begin{definition}[Coverage Score]
\label{def:coverage_score}
The coverage score of a set \( S \subseteq D_s \), denoted \( c(S) \), is defined as:
\[
c(S) = \sum_{t=1}^T \max_{\bx \in S} f_t(\bx),
\]
where \( f_t(\bx) \) is the observed value of objective \( t \) at point \( \bx \).
\end{definition}

\begin{definition}[Optimal Covering Set]
\label{def:optimal_covering_set}
Let \( S^{*}_{D_s} \subseteq D_s \) denote the optimal covering set of size \( K \):
\[
S^{*}_{D_s} = \argmax_{S \subseteq D_s, |S| = K} c(S).
\]
Its coverage score is given by \( c(S^{*}_{D_s}) \).
\end{definition}

\begin{definition}[Contribution of Objective Function $f_t$ to Coverage Score]
\label{def:contribution}
Given a covering set \( S \subseteq D_s \), we denote the contribution of objective $f_t \in \{f_1, f_2, \ldots, f_T \}$ to the overall coverage score $c(S)$ as $g_t(S)$ where 
\[
g_t(S) = \max_{\bx \in S} f_{t}(\bx).
\]
It follows that from \cref{def:contribution} that:
\[
c(S) = \sum_{t=1}^T g_t(S).
\] 
\end{definition}

\begin{lemma}[Monotonicity]
\label{lem:monotonicity}
The coverage score $c(S)$ as defined in \cref{def:coverage_score} is monotone, i.e., for any $S \subseteq S' \subseteq D_s$, we have:
\[
c(S) \leq c(S').
\]
\end{lemma}

\begin{proof}
Adding more points to a set can only increase or maintain the maximum values of $f_t$ for each objective $t$, since for each objective:
\begin{equation*}
    g_t(S') = \max\left(g_t(S), g_t(S' \setminus S)\right) \geq g_t(S).
\end{equation*}
Hence, $c(S)$ is monotone.
\end{proof}

\begin{lemma}[Submodularity]
\label{lem:submodularity}
The coverage score $c(S)$ as defined in \cref{def:coverage_score} is submodular, i.e., for any $S \subseteq S' \subseteq D_s$ and any $\bx^* \in D_s \setminus S'$, we have:
\[
c(S \cup \{\bx^*\}) - c(S) \geq c(S' \cup \{\bx^*\}) - c(S').
\]
\end{lemma}

\begin{proof}
The marginal improvement of adding $\bx^*$ to $S$ is:
\begin{align*}
    c(S \cup \{x^*\}) - c(S) &= \sum_{t=1}^{T} \left( g_t(S \cup \{x^*\}) - g_t(S) \right) \\
    &= \sum_{t=1}^{T} \max\left[f_t(x^*), g_t(S) \right] - g_t(S) \\
    &= \sum_{t=1}^{T} \begin{cases} f_t(\bx^*) - g_t(S) & f_t(\bx^*) > g_t(S) \\ 0 & \textrm{otherwise} \end{cases}.
\end{align*}
Likewise, for $S'$ this is:
\begin{equation*}
    c(S' \cup \{x^*\}) - c(S') = \sum_{t=1}^{T} \begin{cases} f_t(\bx^*) - g_t(S') & f_t(\bx^*) > g_t(S') \\ 0 & \textrm{otherwise} \end{cases}.
\end{equation*}
Now, noting that for each term
\begin{equation*}
    g_t(S') = \max\left[g_t(S), g_t( S' \setminus S) \right] \geq g_t(S),
\end{equation*}
we know then that $f_t(\bx^*) - g_t(S') \leq f_t(\bx^*) - g_t(S)$, and further that $f_t(\bx^*) > g_t(S')$ implies that $f_t(\bx^*) > g_t(S)$. Taken together, these imply that:
\begin{equation*}
    \sum_{t=1}^{T} \begin{cases} f_t(\bx^*) - g_t(S') & f_t(\bx^*) > g_t(S') \\ 0 & \textrm{otherwise} \end{cases} \leq \sum_{t=1}^{T} \begin{cases} f_t(\bx^*) - g_t(S) & f_t(\bx^*) > g_t(S) \\ 0 & \textrm{otherwise} \end{cases}
\end{equation*}
and therefore:
\begin{equation*}
    c(S \cup \{\bx^*\}) - c(S) \geq c(S' \cup \{\bx^*\}) - c(S').
\end{equation*}
\end{proof}

\begin{lemma}[Greedy Achieves a $(1 - \frac{1}{e})$-Approximation for Monotone Submodular Functions]
\label{lem:any-monotone-submodular-func}
For any monotone submodular function, the greedy submodular optimization strategy provides a $(1 - \frac{1}{e})$-approximation. 
\end{lemma}

\begin{proof}
This is a well-known result about monotone submodular functions shown, for example, by \citet{greedysubmodular}. 
\end{proof}

\begin{proof}
\textbf{Proof of \cref{theorem:greedy-simple-approx}}
From \cref{lem:monotonicity} and \cref{lem:monotonicity}, we know that $c(S)$ is monotone submodular. 

Since $c(S)$ is monotone submodular, and \cref{alg:greedy-simple} approximates $A^{*}_{D_s}$ using greedy submodular optimization, it follows from \cref{lem:any-monotone-submodular-func} that \cref{alg:greedy-simple} achieves:
\[
c(A^{*}_{D_s}) \geq \left(1 - \frac{1}{e}\right)c(S^{*}_{D_s}).
\]
\end{proof}

%%%%%%%%%%%%%%%%%%%%%%%%%%%%%%%%%%%%%%%%%%%%%%%%%%%%%%%%%%%%

\end{document}